\documentclass{article}

\usepackage{microtype}
\usepackage{graphicx}
\usepackage{subcaption}
\usepackage{booktabs} % for professional tables
\usepackage{makecell}
\usepackage{hyperref}

\usepackage[preprint]{icml2026}

\usepackage{amsmath}
\usepackage{amssymb}
\usepackage{mathtools}
\usepackage{amsthm}

\usepackage[capitalize,noabbrev]{cleveref}

\theoremstyle{plain}
\newtheorem{theorem}{Theorem}[section]

\theoremstyle{definition}

\theoremstyle{remark}

\usepackage[textsize=tiny]{todonotes}

\usepackage{enumitem}
\setlist[itemize]{itemsep=1pt, topsep=0pt, parsep=0pt, partopsep=0pt}
\setlist[enumerate]{leftmargin=*, itemsep=1pt, topsep=0pt, parsep=0pt, partopsep=0pt}
\usepackage{multirow}

\renewcommand{\paragraph}[1]{\vspace{1.25mm}\noindent\textbf{#1}}

\definecolor{stfred}{RGB}{140, 21, 21}
\usepackage{tcolorbox}

\usepackage{dsfont}

\usepackage{colortbl}
\usepackage[table]{xcolor}

\usepackage{siunitx}

\newcommand{\R}{\mathbb{R}}
\renewcommand{\P}{\mathbb{P}}

\newcommand{\x}{\mathbf{x}}
\newcommand{\h}{\mathbf{h}}

\newcommand{\y}{\mathbf{y}}
\newcommand{\z}{\mathbf{z}}

\newcommand{\q}{\mathbf{q}}
\renewcommand{\k}{\mathbf{k}}

\renewcommand{\H}{\mathbf{H}}
\newcommand{\Q}{\mathbf{Q}}
\newcommand{\K}{\mathbf{K}}
\newcommand{\V}{\mathbf{V}}
\renewcommand{\O}{\mathbf{O}}
\newcommand{\F}{\mathcal{F}}

\newcommand{\A}{\mathbf{A}}
\newcommand{\Y}{\mathbf{Y}}

\newcommand{\U}{\mathbf{U}}

\renewcommand{\S}{\mathbf{S}}

\newcommand{\softmax}{\operatorname{softmax}}
\newcommand{\Concat}{\operatorname{Concat}}
\newcommand{\SiLU}{\operatorname{SiLU}}
\newcommand{\RMSNorm}{\operatorname{RMSNorm}}

\newcommand{\T}{\mathrm{T}}

\renewcommand{\d}[1]{\mathrm{d}_{\mathrm{#1}}}
\newcommand{\N}[1]{\mathrm{N}_{\mathrm{#1}}}
\newcommand{\W}[1]{\mathbf{W}_{\mathrm{#1}}}
\newcommand{\M}[1]{\mathbf{M}_{\mathrm{#1}}}

\icmltitlerunning{Anatomy of Massive Activations and Attention Sinks}

\begin{document}

\twocolumn[
  \icmltitle{The Spike, the Sparse and the Sink: \\
    Anatomy of Massive Activations and Attention Sinks}



  \icmlsetsymbol{equal}{*}

  \begin{icmlauthorlist}
    \icmlauthor{Shangwen Sun}{nyu}
    \icmlauthor{Alfredo Canziani}{nyu}
    \icmlauthor{Yann LeCun}{nyu}
    \icmlauthor{Jiachen Zhu}{nyu}
  \end{icmlauthorlist}

  \icmlaffiliation{nyu}{New York University}

  \icmlcorrespondingauthor{Shangwen Sun}{shangwen.sun@nyu.edu}
  \icmlcorrespondingauthor{Jiachen Zhu}{jiachen.zhu@nyu.edu}

  \icmlkeywords{Machine Learning, ICML}

  \vskip 0.3in
]



\printAffiliationsAndNotice{}  % no special notice (required even if empty)

\begin{abstract}
We study two recurring phenomena in Transformer language models: \emph{massive activations}, in which a small number of tokens exhibit extreme outliers in a few channels, and \emph{attention sinks}, in which certain tokens attract disproportionate attention mass regardless of semantic relevance.
Prior work observes that these phenomena frequently co-occur and often involve the same tokens, but their functional roles and causal relationship remain unclear.
Through systematic experiments, we show that the co-occurrence is largely an \emph{architectural artifact} of modern Transformer design, and that the two phenomena serve related but distinct functions.
Massive activations operate \emph{globally}: they induce near-constant hidden representations that persist across layers, effectively functioning as implicit parameters of the model.
Attention sinks operate \emph{locally}: they modulate attention outputs across heads and bias individual heads toward short-range dependencies.
We identify the pre-norm configuration as the key choice that enables the co-occurrence, and show that ablating it causes the two phenomena to decouple.
\end{abstract}

\section{Introduction}

Transformer-based~\cite{vaswani2017attention} Large language models (LLMs) have achieved unprecedented success across a wide range of tasks~\cite{radford2018improving,radford2019language,brown2020language,openai2022chatgpt,achiam2023gpt,touvron2023llama,touvron2023llama2,grattafiori2024llama,qwen25technicalreport,yang2025qwen3}, yet many aspects of their internal computations remain poorly understood.
In this paper, we study two phenomena that reliably co-occur in decoder-only, pre-norm Transformers~\cite{radford2018improving,xiong2020layer}: \emph{massive activations}, in which a handful of tokens exhibit extreme outliers in a few hidden channels~\cite{sun2024massive,yu2024super}, and \emph{attention sinks}, in which a small number of tokens attract disproportionate attention mass across many heads and layers~\cite{xiao2023efficient}.
Both phenomena have significant practical implications for quantization~\cite{xiao2023smoothquant,yu2024super,son2024prefixing}, pruning~\cite{ma2023llm,sandoval2025using,shin2025orthorank}, KV-cache management~\cite{ge2023model,su2025kvsink,wu2024layer}, and long-context inference~\cite{huang2023advancing,fu2025h,xiao2024duoattention}, among others.
Understanding how these two phenomena relate is therefore both theoretically and practically important.

Prior work~\cite{sun2024massive,kaul2024attention,queipo2025attention} has suggested that the co-occurrence is driven by the overlap of involved tokens, but existing explanations remain largely descriptive.
Here, we move beyond description to provide a \emph{mechanistic account} of \emph{how} and \emph{why} this overlap emerges in pretrained LLMs.
Our core finding is that the co-occurrence is not an inherent property of Transformers, but a predictable consequence of specific architectural and training choices.

We advance three central claims:
\textbf{First}, normalization is a key architectural component bridging the relationship between massive activations and attention sinks.
Changing the normalization configuration can suppress massive activations while preserving attention sinks.
Mechanistically, massive activations interact with normalization to produce near-constant hidden representations within a forward pass, effectively serving as implicit parameters that can be exploited to generate attention sinks.
\textbf{Second}, attention sinks are primarily driven by the dimensionality of the attention space and by the training context-length distribution.
We further show that sinks provide a mechanism to dynamically modulate attention output across heads, biasing certain heads toward \emph{short-range dependencies} that capture local sentence structure.
\textbf{Third}, each phenomenon can be independently suppressed without degrading language-modeling performance, suggesting that their overlap reflects incidental architectural interactions rather than a functional necessity.

Together, our results clarify the casual relationship between massive activations and attention sinks and show how alternative design choices can mitigate either of the phenomena.
For readability, we refer to the tokens and channels that exhibit massive activations as \emph{spike tokens} and \emph{spike channels}, and to the tokens and attention heads affected by attention sinks as \emph{sink tokens} and \emph{sink heads}.

\section{Preliminaries}

Modern LLMs are typically trained on the next-token prediction task \cite{bengio2003neural, radford2018improving} over large text corpora, using decoder-only Transformers \cite{vaswani2017attention} with the pre-norm configuration \cite{xiong2020layer}.
Despite variations in model size and training data, these choices have remained remarkably consistent across most major models.
This section formalizes these core components and establishes the notation used throughout subsequent discussion.

\subsection{Next-Token Prediction}

Next-token prediction is a self-supervised learning objective that leverages the sequential structure of natural language.
By treating token order as a natural supervisory signal, models can be trained on vast unlabeled corpora.
Formally, let $\x \coloneqq (x_1, \ldots, x_{\T})$ be a sequence of $\T$ tokens, where each token $x_i$ takes values in a finite vocabulary $\mathcal{V}$.
A language model parameterized by $\theta$ defines a joint distribution:
\begin{align}
\P_\theta(\x) = \P_\theta(x_1, \ldots, x_{\T}).
\end{align}
Direct modeling of this joint distribution is computationally intractable due to the exponential growth of the sample space.
Autoregressive models address this by factorizing the joint distribution into a product of conditional probabilities:
\begin{align}
\P_\theta(x_1, \ldots, x_{\T}) = \prod_{i} \P_\theta(x_i \mid \x_{<i}),
\end{align}
where $\x_{<i}\coloneqq(x_1, \dots, x_{i-1})$ represents the prefix (context) preceding index $i$.

In decoder-only Transformers, each conditional is computed by mapping the prefix $\x_{<i}$ to a distribution over $\mathcal{V}$.
During training, all conditionals are produced in parallel by supplying the ground-truth prefix at every position via teacher forcing \cite{williams1989learning}.
Given a training corpus $\mathcal{D}$, parameters $\theta$ are learned by minimizing the expected negative log-likelihood:
\begin{align}
\mathcal{L}(\theta) \coloneqq - \mathbb{E}_{\x \sim \mathcal{D}} \left [ \sum_{i} \log \P_\theta(x_i \mid \x_{<i}) \right ].
\end{align}
This objective reduces language modeling to a sequence of conditional classification problems over $\mathcal{V}$, with the conditioning context growing with $i$.

\subsection{Transformer Architecture}

Since the introduction of the Transformer model, many architectural variants have been proposed, and modern LLMs differ in numerous details.
In this section, we describe the specific architecture used by the Llama family of LLMs~\citep{touvron2023llama,touvron2023llama2,grattafiori2024llama}. We focus on Llama because it is among the most widely used open-weight models, and its design choices have strongly influenced subsequent open models such as Qwen~\citep{qwen25technicalreport,yang2025qwen3} and Mistral~\citep{liu2026ministral}.

\paragraph{Token embedding.}
A natural language sentence is first decomposed into a sequence of discrete tokens by a tokenizer, then mapped to continuous vectors via an embedding table.
Specifically, each token is mapped to a $\d{model}$-dimensional vector. For a sequence of $\T$ tokens, we denote the resulting hidden representation by $\H_1 \in \R^{\T \times \d{model}}$.

\paragraph{Transformer layers.}
Starting from $\H_1$, a stack of $\mathrm{L}$ Transformer layers transforms the hidden representation while preserving its dimensionality. Each layer consists of two blocks—an attention block and a feed-forward block—yielding $2\mathrm{L}$ blocks in total.

Let $\H_i \in \R^{\T \times \d{model}}$ denote the input to block $i$, and let $\F_i(\cdot)$ denote its transformation.
Every block employs a residual connection with pre-norm configuration:
\begin{align}
\label{equation:residual_update}
\H_{i+1} = \H_i + \F_i(\RMSNorm(\H_i)),
\end{align}
where $\F_i$ is the attention block when $i$ is odd and the feed-forward block when $i$ is even.
The function $\RMSNorm(\cdot)$ \cite{zhang2019root} is applied row-wise:
\begin{align}
\RMSNorm(\h) \coloneqq \sqrt{\d{model}} \frac{\h}{\|\h\|},
\end{align}
where $\h \in \R^{\d{model}}$ is a single row of $\H_i$.
We omit the learnable scale parameter from the $\RMSNorm$ formulation here, since every $\RMSNorm$ is immediately followed by a linear layer and the scale parameter can be absorbed into the subsequent weight matrix during the forward pass.

\paragraph{Attention block.}
The attention mechanism is implemented as multi-head attention with $\N{head}$ heads, each of dimension $\d{head}$.
For each head $i$, the normalized input $\tilde{\H} \coloneqq \RMSNorm(\H)$ is projected using head-specific weight matrices $\W{Q}^{(i)}, \W{K}^{(i)}, \W{V}^{(i)} \in \R^{\d{model} \times \d{head}}$:
\begin{align}
    \Q^{(i)} &\coloneqq \tilde{\H} \W{Q}^{(i)}, \\
    \K^{(i)} &\coloneqq \tilde{\H} \W{K}^{(i)}, \\
    \V^{(i)} &\coloneqq \tilde{\H} \W{V}^{(i)}, \\
    \A^{(i)} &\coloneqq \softmax\!\left(\frac{\Q^{(i)} {\K^{(i)\top}}}{\sqrt{\d{head}}} + \M{causal}\right), \\
    \O^{(i)} &\coloneqq \A^{(i)} \V^{(i)},
    \label{equation:head_output}
\end{align}
where $\softmax$ is applied row-wise to ensure each row of $\A^{(i)}$ forms a valid probability distribution.
The causal mask $\M{causal} \in \R^{\T \times \T}$ enforces the autoregressive property: its entries are $0$ on and below the diagonal and $-\infty$ above, preventing each position from attending to future tokens.
For simplicity, we omit positional encoding from our description.
In practice, Llama applies Rotary Position Embeddings \cite{su2024roformer} to the $\Q^{(i)}$ and $\K^{(i)}$ before computing $\A^{(i)}$.

The per-head outputs are concatenated and projected via $\W{O} \in \R^{(\N{head} \cdot \d{head}) \times \d{model}}$ to produce the final output:
\begin{align}
\label{equation:attn_output}
\F_{\mathrm{attn}}(\tilde{\H}) \coloneqq \Concat\!\left(\O^{(1)}, \dots, \O^{(\N{head})}\right) \W{O}.
\end{align}

\paragraph{Feed-forward block.}
While the attention block facilitates information exchange across token positions, the feed-forward block operates independently on each position.
Modern LLMs typically employ the SwiGLU activation function~\cite{shazeer2020glu}.
For an input vector $\tilde{\h} \in \R^{\d{model}}$ (a row of $\tilde{\H}$), the feed-forward transformation is defined as:
\begin{align}
\label{equation:SwiGLU}
\F_{\mathrm{ffn}}(\tilde \h) \coloneqq \W{down} \cdot \left(\SiLU(\W{gate} \tilde \h) \odot (\W{up} \tilde \h) \right ),
\end{align}
where $\odot$ denotes the element-wise (Hadamard) product.
The weight matrices are the gate-projection $\W{gate} \in \R^{\d{ffn} \times \d{model}}$, the up-projection $\W{up} \in \R^{\d{ffn} \times \d{model}}$, and the down-projection $\W{down} \in \R^{\d{model} \times \d{ffn}}$.
Here $\d{ffn}$ denotes the intermediate dimension, which is typically three or four times large than $\d{model}$.

\paragraph{Prediction head.}
After all $2L$ blocks, the final hidden representation passes through a $\RMSNorm$ layer and a linear projection to produce logits for next-token prediction:
\begin{align}
    \Y \coloneqq \RMSNorm(\H_{2L+1})\, \W{head},
\end{align}
where $\H_{2L+1}$ is the output of the last residual block, $\W{head} \in \R^{\d{model} \times |\mathcal{V}|}$ is the projection head, and $\Y \in \R^{\T \times |\mathcal{V}|}$ is the matrix of output logits.

\section{From Spikes to Sinks}
\label{section:from_spikes_to_sinks}

This section examines the co-emergence of massive activations and attention sinks in pretrained LLMs.
We begin by tracing the formation of massive activations, identifying the architectural components responsible for their generation and propagation.
We then show how normalization transforms these tokens with massive activations into sparse, near-constant input vectors, enabling the formation of attention sinks.
Our analysis draws primarily from the Llama~\cite{touvron2023llama,touvron2023llama2,grattafiori2024llama} and Qwen~\cite{qwen25technicalreport,yang2025qwen3} model families.
We present the main findings and key empirical results here, deferring additional supporting evidence to~\cref{appendix:additional_empirical_results}.

\subsection{The Emergence of Massive Activations}

Prior work~\cite{bondarenko2023quantizable,xiao2023smoothquant,nrusimha2024mitigating,sun2024massive,yu2024super} has characterized massive activations as extreme outliers in the post-residual hidden representations of Transformers.
These outliers, which often exceed typical activation scales by several orders of magnitude, exhibit five recurring properties across models and prompts:
\begin{enumerate}[label=(\roman*)]
\item They appear only in \emph{intermediate layers}.\label{property:1}
\item They appear only in a \emph{small number of channels}.\label{property:2}
\item The affected channels \emph{consistently spike together}.\label{property:3}
\item The spikes maintain \emph{almost fixed inter-channel ratios}.\label{property:4}
\item They appear only in a \emph{small number of tokens}.\label{property:5}
\end{enumerate}
In this subsection, we trace the emergence of massive activations systematically, showing how each of these properties arises.
As we will see in the next subsection, these properties play a foundational role in enabling attention sinks.

\subsubsection{The Life Cycle of Massive Activations}

\begin{figure*}[t]
  \begin{center}
    \begin{subfigure}[b]{.49\textwidth}
        \centerline{\includegraphics[height=0.65\linewidth]{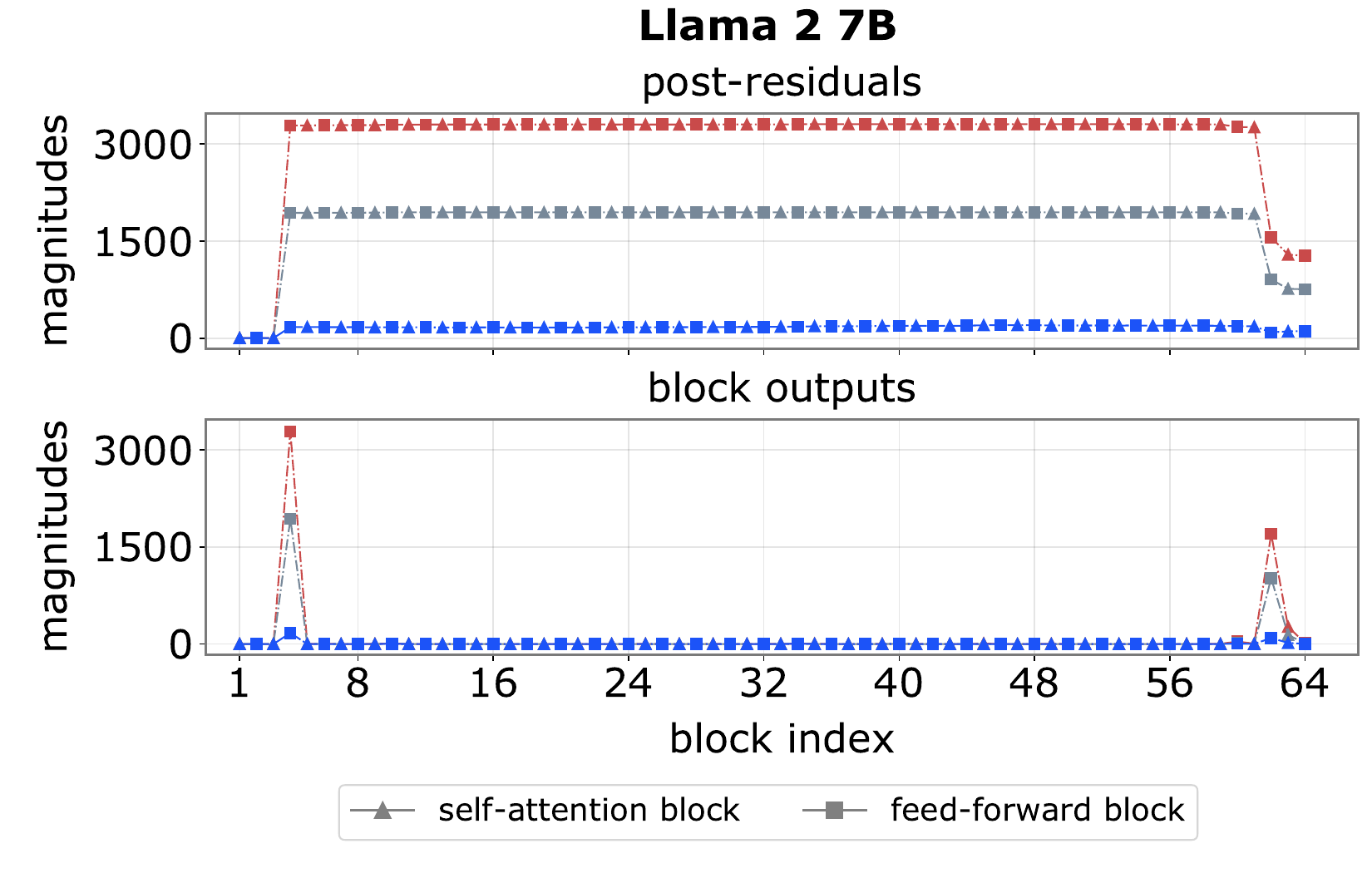}}
    \end{subfigure}
    \hfill
    \begin{subfigure}[b]{.49\textwidth}
        \centerline{\includegraphics[height=0.65\linewidth,trim={1.6cm 0 0 0},clip]{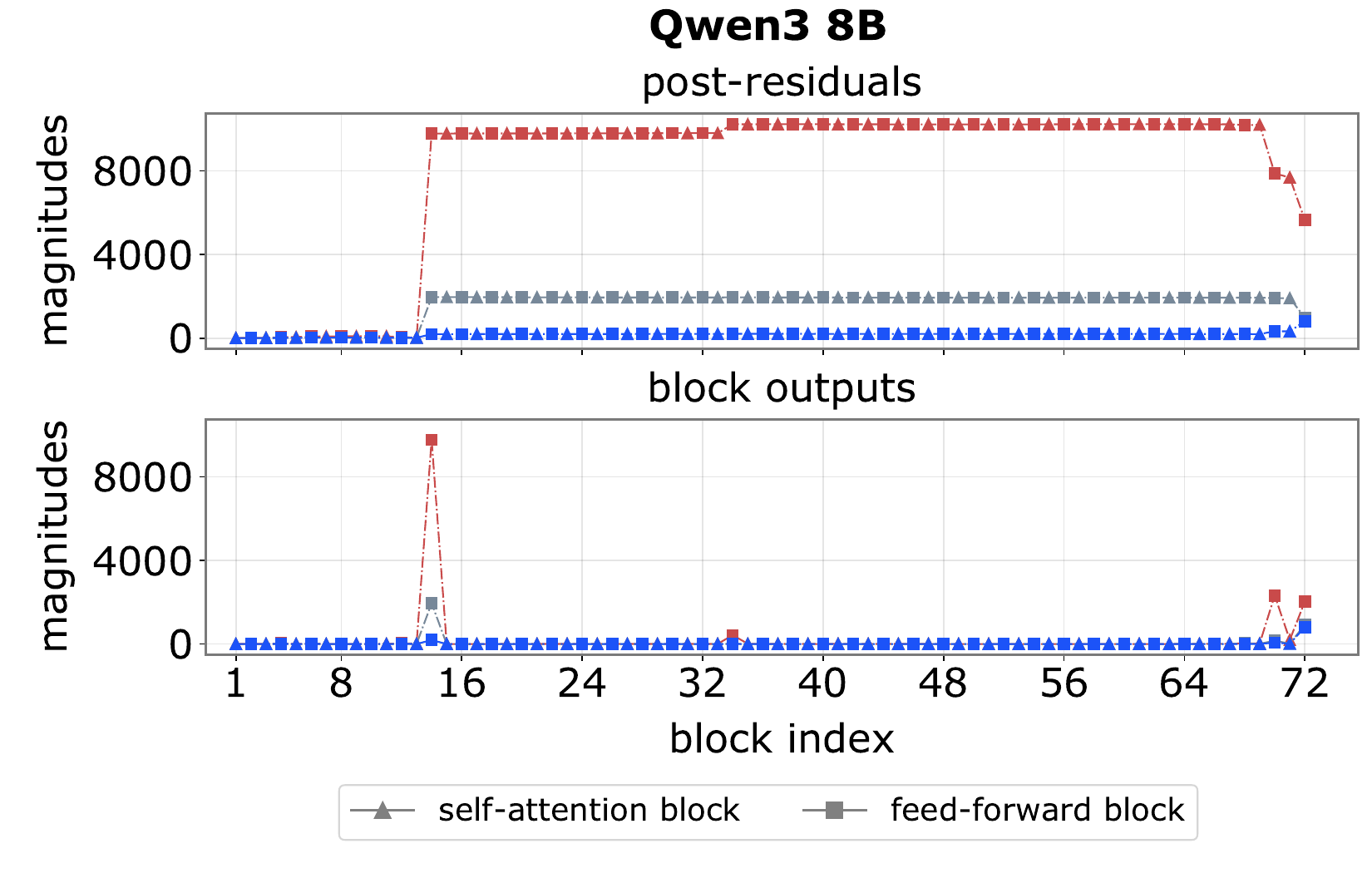}}
    \end{subfigure}
    \vspace{-5pt}
    \caption{\textbf{Top-3 channel magnitudes across depth in Llama 2 7B and Qwen3 8B (post-residuals vs.\ block outputs).} In both models, early blocks inject massive activations that persist through most of the network before being neutralized by late blocks.}
    \label{figure:massive_activations}
  \end{center}
\vspace{-18pt}
\end{figure*}

To characterize how massive activations vary with depth, we track the top-3 channel magnitudes of post-residual hidden representations (\cref{figure:massive_activations}, top panels), following \citet{sun2024massive}.
The magnitudes follow a ``rise--plateau--fall'' trajectory: a sharp increase in early blocks, a long plateau through intermediate blocks, and an abrupt return to typical magnitudes near the end.
This suggests a three-stage life cycle: 
(1) early blocks inject extreme values into the hidden representations;
(2) intermediate blocks propagate these values via the residual connection; and
(3) late blocks neutralize them by inject extreme values with opposite sign.
We describe each stage in turn.

\paragraph{Step-up blocks.}
By examining the individual block outputs (\cref{figure:massive_activations}, bottom panels), we find that massive activations are reliably introduced by \emph{one or two early blocks}, which we term \emph{step-up blocks}.
Prior to these blocks, spike tokens have magnitudes comparable to standard tokens.
The step-up blocks produce extreme values in the spike channels, which are then added to the hidden representation via the residual connection, creating the massive activations.

\paragraph{Residual accumulation.}
In pre-norm Transformers, the hidden representation at depth $i$ can be expressed by unrolling the recurrence in \cref{equation:residual_update}:
\begin{align}
\H_{i+1} = \H_1 + \sum_{j=1}^{i} \F_j(\RMSNorm(\H_j)).
\end{align}
Because the residual stream is additive, extreme values injected by any block $\mathcal{F}_j$ persist through all subsequent blocks unless explicitly counteracted. Empirically, intermediate block contributions to spike channels are typically two to three orders of magnitude smaller than the massive activations themselves. As a result, the massive activations introduced by step-up blocks dominate the residual stream until a later block cancels them.

\paragraph{Step-down blocks.}
As shown in \cref{figure:massive_activations}, massive activations consistently disappear near the end of the network.
Symmetrically to step-up blocks, we identify \emph{one or a few late blocks}, termed \emph{step-down blocks}, whose outputs match the massive activations in magnitude but carry the opposite sign in the corresponding channels.
By contributing an additive inverse via the residual connection, these blocks neutralize the massive activations, returning the hidden representation to a standard range.

\cref{table:step_up_and_step_down} summarizes the step-up and step-down block indices across models.
The consistent positioning of step-up blocks near the beginning and step-down blocks near the end directly accounts for \textbf{Property~\ref{property:1}}: massive activations are confined to intermediate layers because they are injected early and systematically neutralized before the final output.

\begin{table}[t]
  \label{table:step_up_and_step_down}
  \caption{\textbf{Step-up and step-down block indices across models.} Step-up blocks appear near the beginning of the network and step-down blocks near the end, confining massive activations to intermediate layers. Odd and even indices denote attention and feed-forward blocks, respectively.}
  \vspace{-4pt}
  \begin{center}
  \begin{sc}
    \begin{small}
        \begin{tabular}{lccc}
        \toprule
        Model & \# Blocks & Step-Up & Step-Down\\
        \midrule
        Llama 2 7B & 64 & 4 & 62 \\
        Llama 2 13B & 80 & 8 & 78, 79 \\
        Llama 3 8B & 64 & 4 & 64 \\
        \midrule
        Qwen2.5 7B  & 56 & 8, 10 & 54, 55 \\
        Qwen2.5 14B & 96 & 10 & 90, 92, 94, 95 \\
        Qwen3 8B & 72 & 14 & 70, 72 \\
        Qwen3 14B & 80 & 14 & 79\\
        \bottomrule
    \end{tabular}
    \end{small}
  \end{sc}
  \end{center}
  \vspace{-28pt}
\end{table}

\begin{figure}[t]
  \begin{center}
    \centerline{\includegraphics[width=\linewidth]{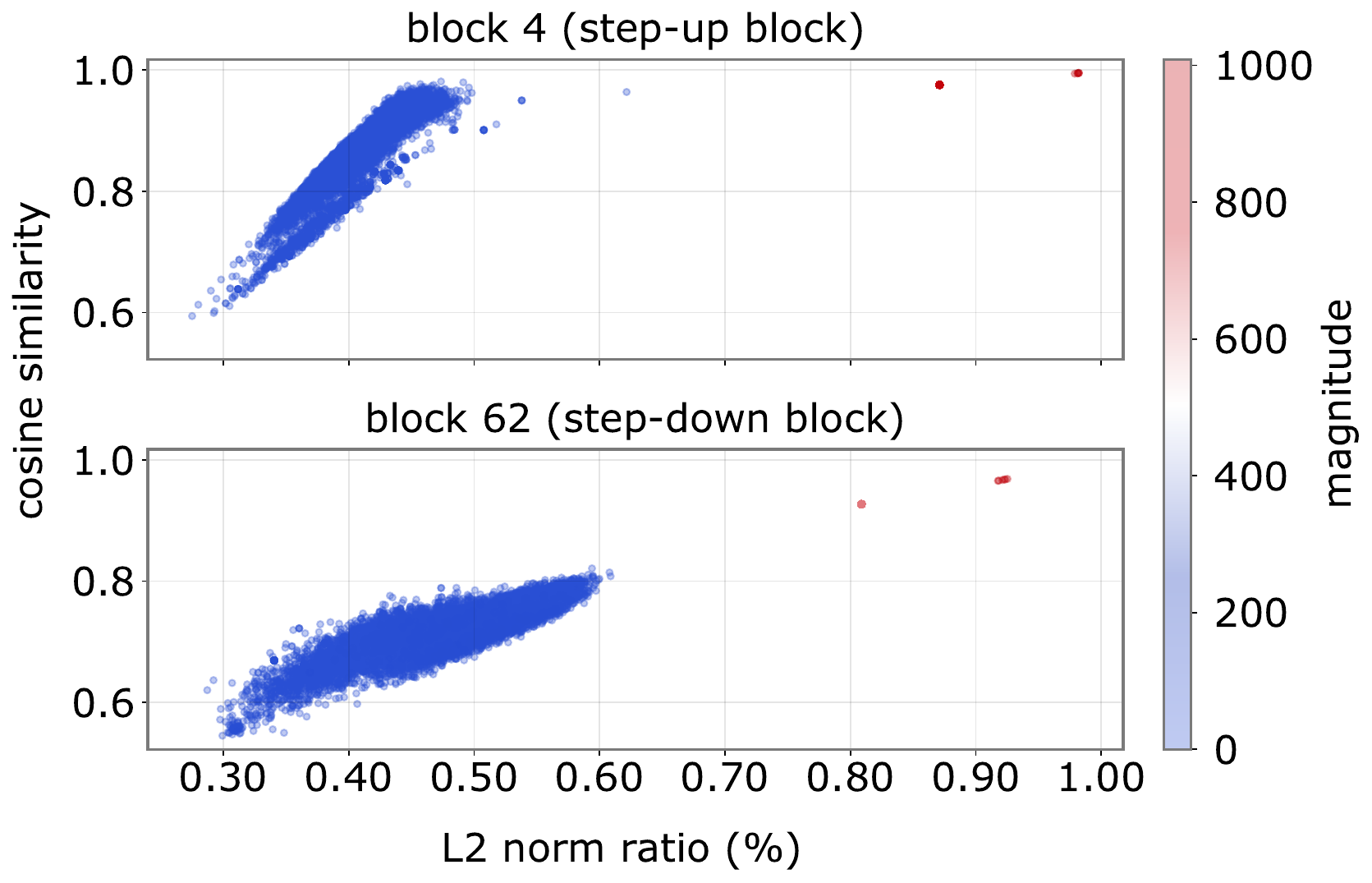}}
    \caption{\textbf{Input-output characteristics of $\SiLU$ in step-up and step-down blocks of Llama 2 7B.} Based on $1024$ randomly sampled sentences from C4 dataset \cite{raffel2020exploring}, we plot the cosine similarity and norm ratio for each token. Points are colored by the maximum magnitude of the block output. For spike tokens (red points), both direction and norm remain largely unchanged, indicating that the $\SiLU$ gate operates in a near-identity regime.}
    \label{figure:cosine_similarities}
  \end{center}
  \vspace{-28pt}
\end{figure}

\subsubsection{Feed-Forward Block as Directional Quadratic Amplifier}

While both attention and feed-forward blocks possess the theoretical capacity to produce large outputs, our analysis reveals that the SwiGLU-based feed-forward block is the primary source of massive activations, functioning as a \emph{directional quadratic amplifier}. We characterize the mechanism by which extreme activations arise for a small subset of tokens in Llama 2 7B.
Other models share the same high-level mechanism, but differ in more precise details that instantiate it. we defer those results to \cref{appendix:additional_empirical_results}.

\paragraph{Near-identity gating regime.}
Let $\tilde{\h}^{(s)} \in \R^{\d{model}}$ denote the normalized input of a spike token to a step-up or step-down feed-forward block.
We empirically observe that the $\SiLU$ nonlinearity operates in a near-identity regime ($\SiLU(\x) \approx \x$), as shown in~\cref{figure:cosine_similarities}.
Under this approximation, the feed-forward transformation reduces to:
\begin{align}
\label{equation:SwiGLU_approximation}
\F_{\text{ffn}}(\tilde{\h}^{(s)}) \!\approx\! \W{down}\!\cdot\!\left( (\W{gate} \tilde{\h}^{(s)})\!\odot\!(\W{up} \tilde{\h}^{(s)}) \right).
\end{align}

\paragraph{High-gain quadratic structure.}
Let $\W{gate}^{(i)}$ and $\W{up}^{(i)}$ denote the $i$-th rows of the respective weight matrices, and let $\W{down}^{(k,i)}$ denote the $(k,i)$-th entry of $\W{down}$. Each output coordinate $k$ then admits the quadratic form (derived in detail in~\cref{theorem:silu_quadratic_form_approx}):
\begin{align}
\F_{\text{ffn}}(\tilde{\h}^{(s)})_k \approx 
\tilde{\h}^{(s)\top} \U_k \tilde{\h}^{(s)} = 
\tilde{\h}^{(s)\top} \S_k \tilde{\h}^{(s)},
\end{align}
where
\begin{align}
\U_k &= \sum_{i=1}^{\d{ffn}} \W{down}^{(k,i)}\, 
        \W{gate}^{(i)} \W{up}^{(i)\top},
        \label{equation:quadratic_form}\\
\S_k &= \tfrac{1}{2}(\mathbf{U}_k + \mathbf{U}_k^\top).
\end{align}

\cref{figure:U_norm} shows the Frobenius norms $\|\U_k\|_F$ across all output coordinates and feed-forward blocks of Llama~2 7B.
Spike channels correspond precisely to coordinates with exceptionally large $\|\U_k\|_F$, and these high-norm coordinates appear exclusively in step-up and step-down blocks.
Inspection of the weight matrices reveals that, for high-gain channels $k$, $\W{down}$ contains anomalously large entries $\W{down}^{(k,i)}$ for certain intermediate dimensions $i$, and the corresponding rows $\W{gate}^{(i)}$ and $\W{up}^{(i)}$ are highly collinear, consistent with prior observations from~\citet{yu2024super}.

\begin{figure}[t]
  \centering
  \vspace{8pt}
  \includegraphics[width=\linewidth,trim={0 0 0 1.0cm},clip]{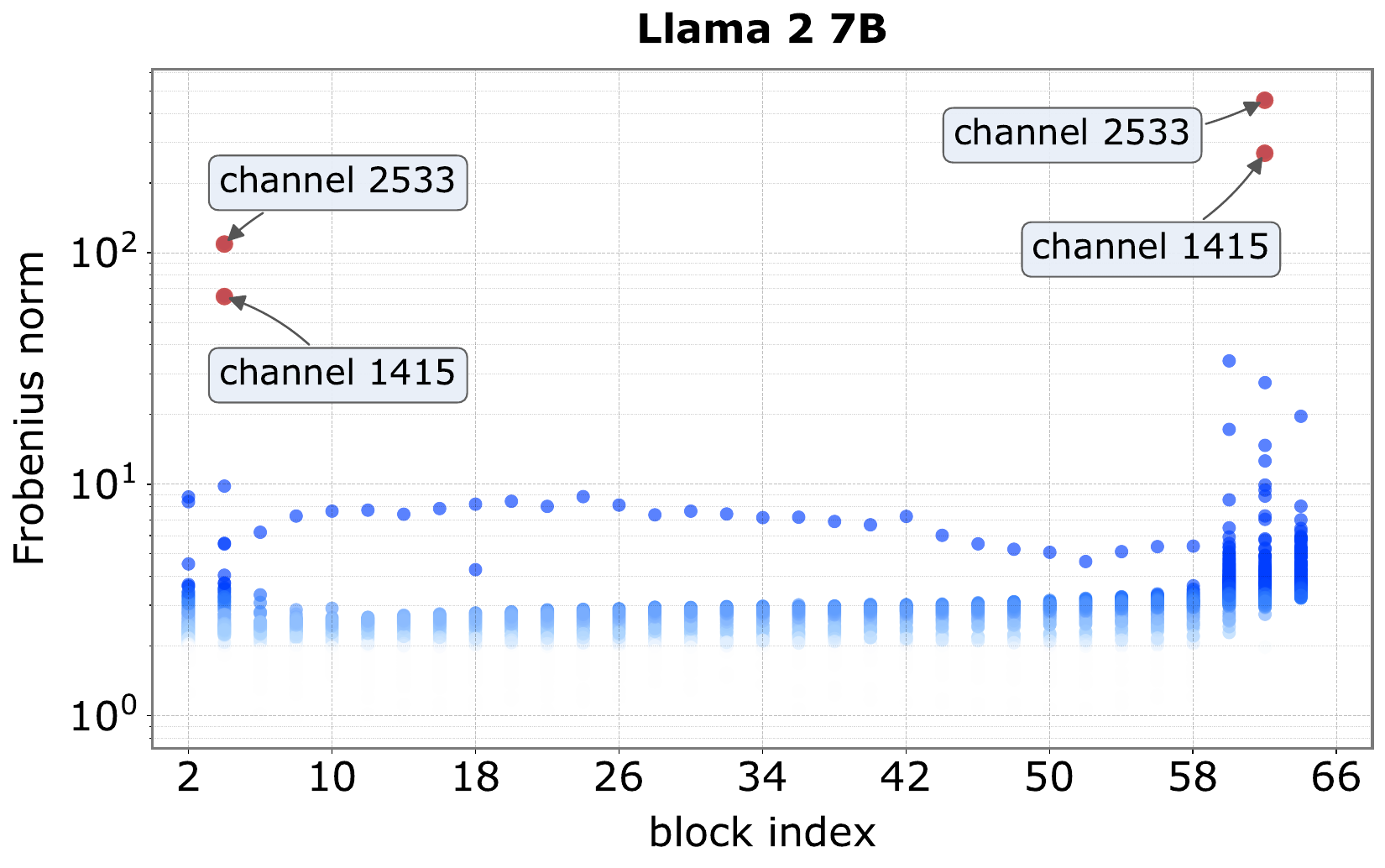}
  \caption{\textbf{Frobenius norms $\|\U_k\|_F$ for the quadratic forms in Llama 2 7B.} Spike channels align with $\U_k$ matrices that have substantially larger norms than typical channels. These high-norm coordinates appear exclusively in step-up and step-down blocks.}
  \label{figure:U_norm}
\vspace{-12pt}
\end{figure}

\paragraph{Rank-one dominance.}
\cref{figure:S_eigenvalues} compares the eigenvalue spectra of $\S_k$ for spike versus non-spike channels. For spike channels, $\S_k$ is dominated by a single eigenvalue $\lambda_\star$ whose magnitude far exceeds the rest of the spectrum. Let $\mathbf{s}_\star$ denote the corresponding unit eigenvector. In such cases, the feed-forward block then acts as a \emph{directional quadratic amplifier} for these channels:
\begin{align}
\F_{\text{ffn}}(\tilde{\h}^{(s)})_k &\approx \tilde{\h}^{(s)\top} \S_k \tilde{\h}^{(s)}\\
&\approx \lambda_\star (\mathbf{s}_\star^\top \tilde{\h}^{(s)})^2 \\
&= \lambda_\star \sqrt{\d{model}} \cos(\mathbf{s}_\star, \tilde{\h}^{(s)}).
\end{align}
When the input $\tilde{\h}^{(s)}$ aligns with the spike direction $\mathbf{s}_\star$, the squared projection is amplified by $\lambda_\star$, producing massive activations. Crucially, inspection of the spike directions across all spike channels reveals that their $\S_k$ matrices share nearly the same principal eigenvector $\mathbf{s}_\star$. Consequently, when an input aligns with this common spike direction, all spike channels are activated simultaneously.

\begin{figure}[t]
  \centering
  \includegraphics[width=0.92\linewidth]{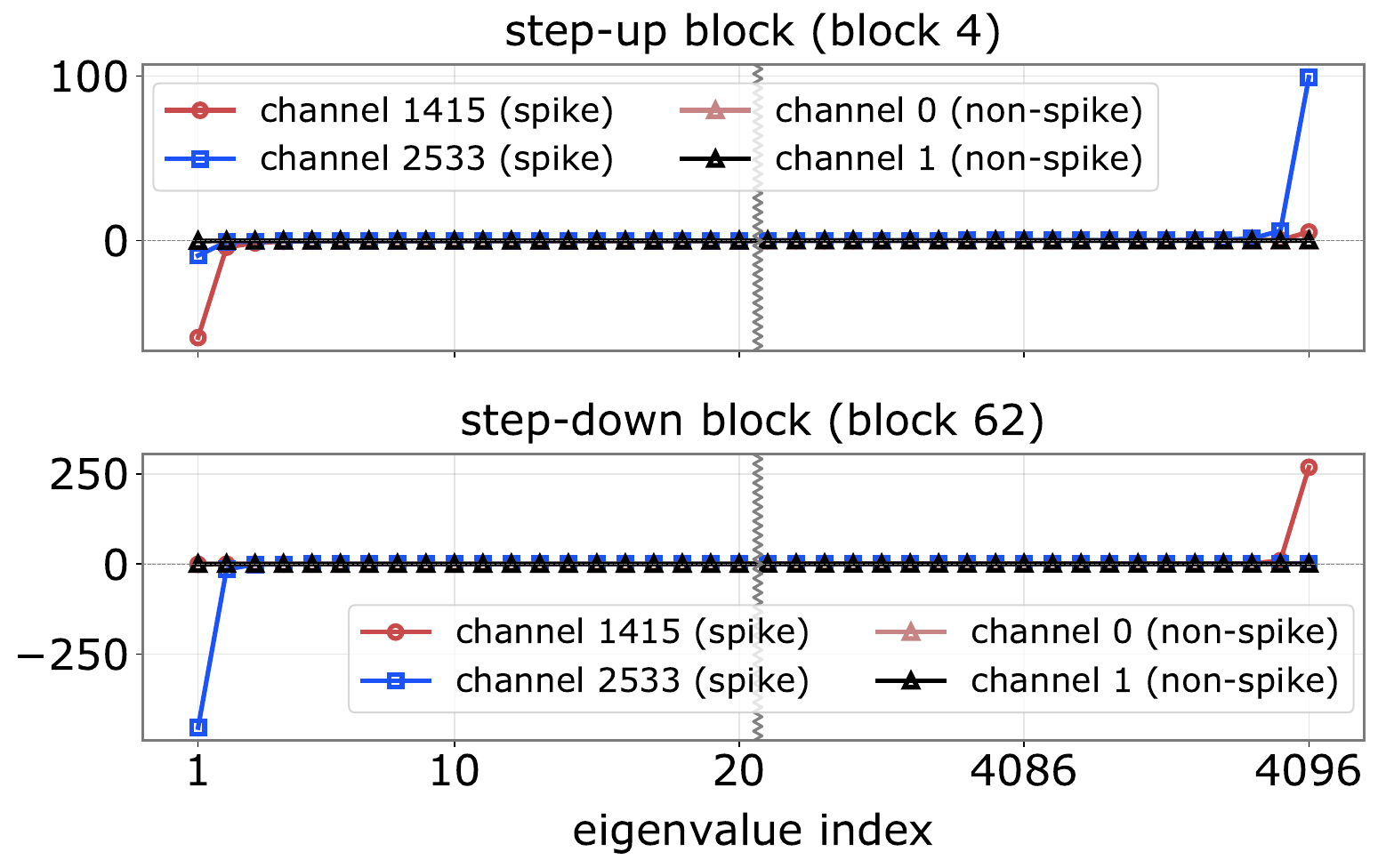}
  \caption{\textbf{Eigenvalue spectra of $\S_k$ for spike vs.\ non-spike channels in Llama 2 7B.} Spike channels exhibit a single dominant eigenvalue $\lambda_\star$ that is orders of magnitude larger than the remainder of the spectrum; non-spike channels show no such outlier.}
  \label{figure:S_eigenvalues}
  \vspace{-18pt}
\end{figure}

This analysis accounts for \textbf{Property~\ref{property:2}}: the scarcity of high-gain quadratic forms explains why massive activations are confined to a small subset of channels.
Furthermore, the existence of a shared spike direction across these channels underpins \textbf{Properties~\ref{property:3},~\ref{property:4} and~\ref{property:5}}; specifically, it accounts for the synchronized triggering of affected channels and their invariant activation magnitude ratios, which are governed by the leading eigenvalues of $\mathbf{S}_k$. 
Finally, because these spike directions are restricted to a highly localized region of the high-dimensional space $\R^{\d{model}}$, extreme activations occur only for tokens whose representations closely align with $\mathbf{s}_\star$. For the vast majority of tokens, the projection onto this direction is negligible.

\subsubsection{What Makes a Token a Spike Token}

While the feed-forward block provides the \emph{capacity} for amplification, it requires the $\tilde{\h}^{(s)}$ to align with the trigger direction $\mathbf{s}_\star$ in order to generate 
massive activations. Prior work \cite{sun2024massive} establishes that spike tokens are almost 
exclusively the first tokens or delimiter tokens; we now study 
why these tokens consistently achieve such alignments.

\paragraph{First tokens.} The initial position serves as the most consistent catalyst for massive activations. A vocabulary-wide probe (\cref{table:first_token}) reveals that over $98\%$ of vocabulary items manifest as spike tokens when placed at position $0$, but rarely do so at subsequent indices. This disparity confirms that the phenomenon is driven by architectural position rather than token semantics. The few exceptions are primarily rare characters from low-resource scripts; we find that their embeddings close to the initialization values, likely due to infrequent gradient updates during pre-training.

\begin{table}[t]
  \caption{\textbf{Ubiquity of initial spikes across diverse LLMs.}
  For nearly all evaluated models, positional occupancy at the initial position induces massive activations in intermediate layers, independent of the token’s semantic identity.}
  \vspace{-8pt}
  \label{table:first_token}
  \begin{center}
  \begin{sc}
    \begin{small}
  \begin{tabular}{lccc}
    \toprule
    Model & \# Vocab & \# Spike Token & Ratio \\
    \midrule
    Llama 2 7B   & 32,000  & 31,887  & 99.65\% \\
    Llama 2 13B  & 32,000  & 31,889  & 99.65\% \\
    Llama 3 8B   & 128,256 & 127,956 & 99.77\% \\
    \midrule
    Qwen 2.5 7B  & 152,064 & 149,587 & 98.40\% \\
    Qwen 2.5 14B & 152,064 & 149,645 & 98.40\% \\
    Qwen 3 8B    & 151,936 & 151,830 & 99.93\% \\
    Qwen 3 14B   & 151,936 & 151,824 & 99.93\% \\
    \bottomrule
  \end{tabular}
  \end{small}
  \end{sc}
  \end{center}
  \vspace{-18pt}
\end{table}

The behavior of the initial position arises because the attention block collapses to a simple linear map. Since the first token only attend to 
itself, its output reduces to:
{\abovedisplayskip=4pt \belowdisplayskip=4pt
\begin{align}
    \F_{\mathrm{attn}}(\h^{(1)}) 
    \!=\! \sum_{i=1}^{\N{head}}  \W{O}^{(i)\top} \W{V}^{(i)\top} \h^{(1)}
    \!\equiv\!  \W{VO}^{\top}\, \h^{(1)},
    \label{equation:attention_output}
\end{align}}where $\h^{(1)} \in \mathbb{R}^{d}$ is the hidden state of the first token, and 
$\W{VO}\coloneqq \sum_{i=1}^{\N{head}} \W{V}^{(i)}\, \W{O}^{(i)}$ is the linear mapping matrix.

In this regime, the attention block applies a \emph{static linear transformation} that is identical across all prompts, consistently steering the first tokens' representations toward the trigger direction $\mathbf{s}_\star$ and thereby inducing the massive activations observed in intermediate layers.

\paragraph{Delimiter Tokens.} Tokens such as periods and newlines follow a mechanistic trajectory similar to first-token sinks. In the early attention blocks, these tokens exhibit significantly elevated post-$\RMSNorm$ magnitudes, stemming from the near-collinearity of their embeddings with the learned scaling parameters of $\RMSNorm$. This magnitude surge induces attention heads to allocate disproportionate weight to the token itself, regardless of the preceding context. So delimiter tokens emulate the isolated environment of the first token across multiple heads. This \emph{self-sinking} behavior allows static linear transformations to project their latent states toward the same high-gain manifold as the first token. Once aligned with $\mathbf{s}_\star$, these representations undergo directional quadratic amplification. 

In summary, a token transitions into a spike token when it demonstrates a strong self-sinking bias in early layers, establishing the stable linear trajectory required to activate the directional quadratic amplifier.
\subsection{The Emergence of Attention Sinks}

Having traced the generation and propagation of massive activations, we now characterize how these spike tokens induce the \emph{attention sink} phenomenon.
Specifically, we demonstrate that normalization transforms spike tokens into sparse, bounded, and nearly constant input vectors, enabling the formation of attention sinks.

\subsubsection{Normalization Transforms Spike Tokens}
In pre-norm Transformer architectures, each attention block operates on normalized hidden representations. 
Let $\h^{(s)}$ denote the hidden representation of a spike token and let $\tilde{\h}^{(s)}$ denote the output of $\RMSNorm\left(\h^{(s)}\right)$. The transformation imparts three properties central to attention sink formation.

\paragraph{Bounded Range.}
Normalization suppresses the extreme magnitudes of spikes, mapping the representation to a bounded range (proof deferred to \cref{theorem:rms_norm_bound}):
\begin{equation}
|\tilde{\h}^{(s)}_i| \leq \sqrt{\d{model}}, \quad \forall\, i \in \{1,\dots,\d{model}\}.
\end{equation}
Hence, even if the pre-norm input contains values on the magnitudes of thousands, the block output $\tilde{\h}^{(s)}$ remains moderate and numerically stable.

\paragraph{Sparsification.}
Because the norm $\|\h^{(s)}\|$ is dominated by a few outlier coordinates, the normalization process effectively suppresses non-spike channels. Consequently, the normalized state $\tilde{\h}^{(s)}$ can be approximated as:
\begin{equation}
\tilde{\h}^{(s)} \approx \sum_{i \in \mathcal{C}} \tilde{\h}^{(s)}_i \mathbf{e}_i,
\end{equation}
where $\mathcal{C}$ denotes the set of spike channel indices and $\mathbf{e}_i$ represents the $i$-th standard basis vector. This transformation yields a sparse, approximately multi-hot representation that is concentrated within a low-dimensional subspace of the original embedding space.

\begin{figure}[t]
    \centering
    \includegraphics[width=\linewidth]{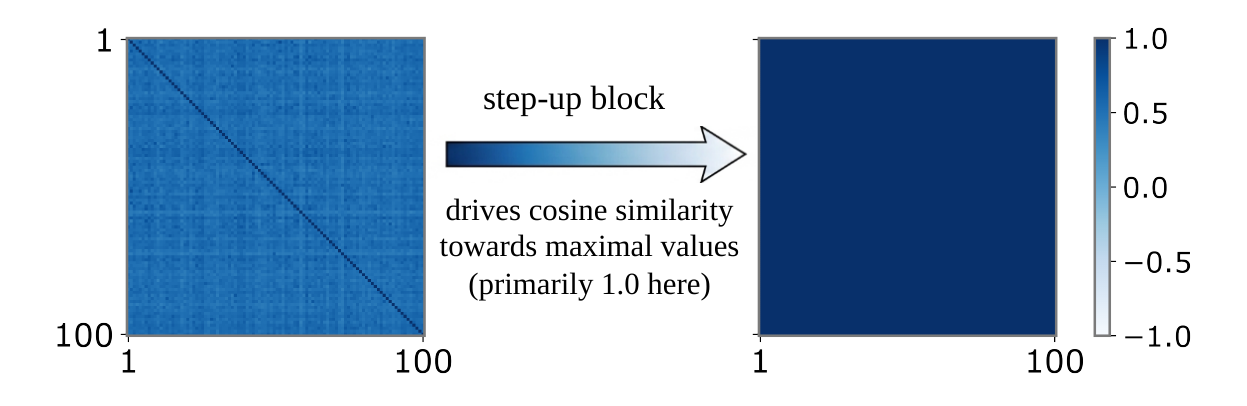}
    \vspace{-5pt}
    \caption{\textbf{Cosine similarity among spike tokens before and after step-up block in Llama 2 7B.} Pre-step-up representations vary across spike tokens, but post-step-up representations collapse to nearly identical directions, empirically validating the near-constant approximation.}
    \label{figure:spike_cos}
    \vspace{-10pt}
\end{figure}

\paragraph{Near-constant vector.}
Spike channels maintain nearly fixed magnitude ratios across spike tokens (\textbf{Property~\ref{property:4}}), so the normalized values $\tilde{\h}^{(s)}_i$ for $i \in \mathcal{C}$ are approximately token-invariant. Consequently, for any spike tokens $a$ and $b$:
\begin{align}
\RMSNorm(\h^{(a)}) \approx \RMSNorm(\h^{(b)}),
\end{align}
even when $\h^{(a)}$ and $\h^{(b)}$ differ substantially in their non-spike channels. Normalization thus collapses distinct representations into a \emph{near-constant} sparse vector, largely erasing token-specific variation. This collapse is empirically demonstrate in~\cref{figure:spike_cos}, where spike tokens following the step-up blocks exhibit cosine similarities approaching $1.0$.

\begin{figure*}[t!]
\centering
\includegraphics[width=\linewidth]{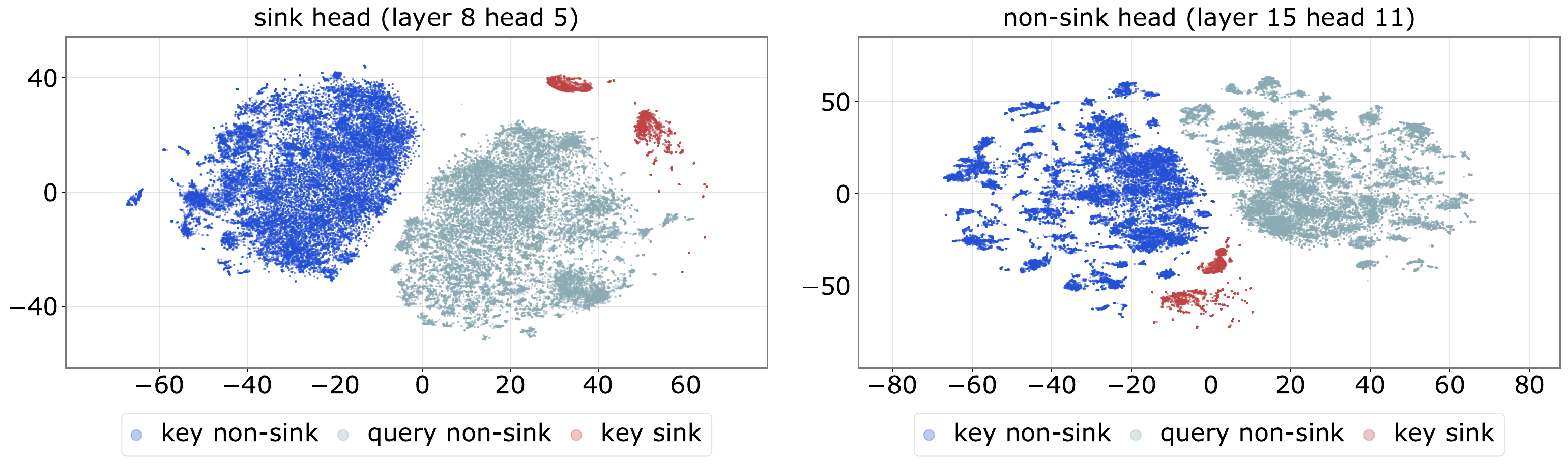}
\caption{\textbf{t-SNE visualization of query and key vectors for a representative sink head (left) and non-sink head (right).} In the sink head, sink keys $\k^{(s)}$ lie closer to $\q^{(n)}$ than non-sink keys $\k^{(n)}$, creating large logit gaps. In the non-sink head, $\k^{(s)}$ and $\k^{(n)}$ are approximately equidistant from $\q^{(n)}$, preventing the formation of a privileged sink position.}
\label{figure:tsne}
\vspace{-10pt}
\end{figure*}

\subsubsection{Geometric Alignment Creates Sinks}

Spike tokens produce sparse normalized representations, which severely restrict the dimensionality of their resulting attention projections. For a given head, the key vector $\mathbf{k}^{(s)}$ of a sink token is given by:
\begin{align}
  \k^{(s)} &=  \W{K}^{\top}\, \tilde{\h}^{(s)}  \approx \sum_{i \in \mathcal{C}} \tilde{\h}^{(s)}_{i}\,  \W{K}^{\top}\, \mathbf{e}_i,
\end{align}
where $\W{K}^\top \mathbf{e}_i$ corresponds to the $i$-th row of the weight matrix. Consequently, the keys $\k^{(s)}$ are confined to the span of only a few rows. 
In practice, we find this subspace typically collapses to only one or two dimensions—a significant reduction compared to the full head dimension $\d{head}$.

Although empirical analysis shows that non-sink queries $\q^{(n)}$ and keys $\k^{(n)}$ also reside in a constrained subspace, their manifold is significantly more expansive than that of the spike tokens. We posit that the emergence of an attention sink is determined not by the absolute volume of these subspaces, but by their \emph{relative geometric alignment}:
\begin{itemize}
    \item \textbf{Sink Heads:} The $\q^{(n)}$ subspace is positioned closer to the fixed $\k^{(s)}$ than to the $\k^{(n)}$ subspace. This alignment produces large, consistent logit gaps in favor of the sink token across diverse inputs.
    \item \textbf{Non-Sink Heads:} The $\q^{(n)}$ subspace is more closely aligned with its non-sink keys $\k^{(n)}$, resulting in attention patterns that distribute mass according to token semantics rather than a fixed default position.
\end{itemize}

As visualized via t-SNE~\citep{maaten2008visualizing} in~\cref{figure:tsne}, the difference between sink and non-sink heads lies in this subspace alignment. In sink heads, the model exploits the near-constant nature of spike keys to create a stable default position for attention mass, effectively offloading excess attention weight to a token whose representation has been neutralized by the normalization function.

Attention sinks arise from two properties of spike tokens after normalization: sparsity and near-constancy. Sparsity restricts sink keys to a low-dimensional subspace (often one or two dimensions) of the row space of $\W{K}$. Near-constancy keeps those keys nearly invariant across prompts. Together, these properties allow the model to reliably separate sink keys from non-sink keys into distinct subspaces, and this separation manifests as the logit gaps characteristic of attention sinks.
\vspace{-6pt}

\subsection{Summary of Findings} 

This section links massive activations and attention sinks through an architecture-driven pathway in pre-norm Transformers.
Massive activations originate from a small number of early \emph{step-up} feed-forward blocks.
In these blocks, SwiGLU behaves as a \emph{directional quadratic amplifier}:
rare high-gain quadratic forms share a common trigger direction, and when a token aligns with it, the token becomes a token carrying massive activations.
Because the residual stream is additive, these outliers persist across intermediate layers.

Normalization then maps spike-token representations to inputs that are sparse and nearly constant.
As a result, diverse spike tokens collapse to a same vector, making their keys low-dimensional and nearly invariant across prompts.
The learned key projection $\W{K}$ consequently maps spike keys and non-spike keys into distinct subspaces.
Attention sinks then emerge in heads whose query subspace aligns more strongly with the fixed sink-key subspace than with the non-sink-key subspace, creating the persistent logit gaps that define attention sinks.
This completes the account of how massive activations and attention sinks co-emerge.
\vspace{-6pt}

\section{Anatomy of Spikes and Sinks}
\label{section:ablations}

The previous section characterize how massive activations and attention sinks co-emerge in pretrained LLMs, suggesting that both phenomena arise from the interactions between architectural components and learned weights.
We now shift from mechanism to causality. Guided by our findings, we perform targeted ablations to identify which architectural and training choices modulate these phenomena, ultimately establishing the causal relationship between the two.

\paragraph{Experimental setup.}
For our baseline setup, we train a Llama-style 7B model~\cite{touvron2023llama,touvron2023llama2,grattafiori2024llama} from scratch on the DCLM dataset~\cite{li2024datacomp} with a budget of 100B tokens. We follow the standard Llama training recipe, supplementing unspecified details with established open-source implementations. Despite training on substantially fewer tokens than the original Llama models, we successfully reproduce the massive activations and attention sinks phenomena described previously. In each ablation, we modify a specific set of settings while keeping the remaining configuration fixed. We evaluate the models using sentences randomly sampled from the C4 dataset~\cite{raffel2020exploring} and report perplexity, sink ratio~\cite{gu2024attention}, and maximal activation magnitudes. Training and evaluation details are provided in \cref{appendix:experimental_settings}.
\begin{table}[t!]
  \caption{\textbf{Optimization Hyperparameters ablations.} We observe that the sink ratio could serve as a proxy for optimization health. Conversely, the magnitude of massive activations varies largely independently of both perplexity and the sink ratio. {\setlength{\fboxsep}{0pt}\colorbox{lime!10}{\strut Highlighted rows denote the baseline result.}}}
  \label{table:training_hyperparameters}
  \begin{center}
    \begin{small}
      \begin{sc}
        \begin{tabular}{lcccr}
          \toprule
            Setup & Perplexity & Sink Ratio & Spike \\
            \midrule
            \multicolumn{4}{c}{Base Learning Rate} \\
            \midrule
            $\num{7.5e-5}$ & 11.8 & 18.6\% & 1447 \\
            $\num{1.5e-4}$ & 10.7 & 31.8\% & 2251  \\
            \rowcolor{lime!10} $\mathbf{\num{3.0e-4}}$ & 10.1 & 46.0\% & 3818 \\
            $\num{6.0e-4}$ & 10.0 & 51.5\% &  3773 \\
            $\num{1.2e-3}$ & 10.2 & 39.2\% &  2723  \\
            \midrule
            \multicolumn{4}{c}{Minimal Learning Rate} \\
            \midrule
            \rowcolor{lime!10} $\mathbf{\num{3e-5}}$ & 10.1 & 46.0\% & 3818 \\
            $\mathbf{\num{3e-4}}$ & 10.7 & 56.8\% & 2870 \\
            \midrule
            \multicolumn{4}{c}{Weight Decay} \\
            \midrule
            $0.0$ & 10.4 & 33.8\% & 12275  \\
            \rowcolor{lime!10} $0.1$ & 10.1 & 46.0\% & 3818 \\
            \midrule
            \multicolumn{4}{c}{AdamW $\beta_2$} \\
            \midrule
            $0.9$ & 10.1 & 49.0\% & 2832 \\
            \rowcolor{lime!10} $0.95$ & 10.1 & 46.0\% & 3818 \\
            $0.999$ & 10.7 & 20.9\% & 1855 \\
            \midrule
            \multicolumn{4}{c}{Training Tokens} \\
            \midrule
            \rowcolor{lime!10} $100$B & 10.1 & 46.0\% & 3818 \\
            $200$B & 9.5 & 63.3\% &  1848\\
          \bottomrule
        \end{tabular}
      \end{sc}
    \end{small}
  \end{center}
  \vspace{-20pt}
\end{table}

\vspace{-15pt}
\subsection{Ablating Optimization Hyperparameters}
Before proceeding to targeted architectural ablations, we examine the sensitivity of these phenomena to common training hyperparameters: learning rate, weight decay, AdamW~\cite{loshchilov2017decoupled} momentum ($\beta_2$), and total training tokens. Results are summarized in \cref{table:training_hyperparameters}.

Two distinct patterns emerge. First, while the sink ratio is not strictly monotonic in perplexity, it remains a robust proxy for \emph{optimization health}. Suboptimal configurations—such as extreme learning rates, disabling weight decay, or mis-specified $\beta_2$—consistently reduce the sink ratio.
Conversely, favorable configurations—such as extending training budget or disabling learning-rate decay—substantially increase the sink ratio.
This suggests that the intensity of attention sinks is tied to the overall optimization health.
Second, the magnitudes of massive activations vary largely independently of perplexity and sink ratio. For instance, disabling weight decay causes activation spikes to exceed $12{,}000$ without any corresponding improvement in sink ratio or perplexity. spikes drive normalized representations into a sparse, near-constant regime, after which further growth in magnitudes contribute diminishingly to attention sinks. Having established that sinks and spikes respond differently to optimization configurations, we now consider architectural interventions that directly target their underlying mechanisms.
\vspace{-5pt}

\subsection{Ablating Massive Activations}

In the previous section, we identify two architectural components that strongly affect the emergence of massive activations: (1) the SwiGLU-based feed-forward network, which generates the massive activations, and (2) the normalization configuration, which governs their propagation and maps spike tokens to sparse, near-constant vectors. In this subsection, we ablate both design choices.

\subsubsection{Feed-Forward Block Design}
Our earlier analysis traced the origin of massive activations to the SwiGLU blocks. To test whether this specific design is a prerequisite for both phenomena, we ablate the feed-forward architecture. Specifically, we evaluate the standard two-layer GeLU-based feed-forward block used in the original Transformer~\cite{vaswani2017attention}, a simplified single linear layer, and an attention-only configuration where all feed-forward blocks are replaced with additional attention layers. Results are summarized in \cref{table:ffn_setup}.
\begin{table}[H]
  \vspace{-10pt}
  \caption{\textbf{Feed-forward block design ablations.} Massive activations and attention sinks emerge across all evaluated designs, including attention-only and single linear configurations. Notably, GeLU and SwiGLU architectures yield significantly higher spike magnitudes by acting as efficient amplifiers. In contrast, linear and attention-only blocks exhibit much lower spikes as they require gradual accumulation across multiple layers. {\setlength{\fboxsep}{0pt}\colorbox{lime!10}{\strut Highlighted rows denote the baseline result.}}}
  \label{table:ffn_setup}
  \begin{center}
    \begin{small}
      \begin{sc}
        \begin{tabular}{lcccr}
          \toprule
            Setup & Perplexity & Sink Ratio & Spike \\
            \midrule
            \multicolumn{4}{c}{Feed-Forward Block} \\
            \midrule
            GeLU & 10.1 & 69.3\% & 3369 \\
            Linear & 12.5 & 58.9\% & 688 \\
            Attention & 10.8 & 73.9\% & 637 \\
            \rowcolor{lime!10} SwiGLU & 10.1 & 46.0\% & 3818 \\
          \bottomrule
        \end{tabular}
      \end{sc}
    \end{small}
  \end{center}
  \vspace{-20pt}
\end{table}
The results indicate that massive activations and attention sinks emerge across all configurations, suggesting that the specific feed-forward block design is not the primary causal driver of either phenomenon. The specific block design is therefore not a prerequisite, but it is a strong modulator of amplification efficiency. SwiGLU and GeLU concentrate outlier growth within a single step, while linear and attention blocks require gradual accumulation across layers.

\subsubsection{Normalization Configuration}

Normalization shapes these phenomena along two axes: how outliers accumulate in the residual stream (governed by the pre-norm configuration), and how spike tokens are transformed into sparse, near-constant vectors (governed by the normalization operator itself). We probe both axes through three variants. First, we test sandwich normalization~\cite{ding2021cogview}, which adds an extra $\text{RMSNorm}$ at the block output, and a variant utilizing QKNorm~\cite{olmo2025olmo}, where input normalization is applied only to queries and keys. Second, we replace standard normalization with an element-wise transformation, DynamicTanh~\cite{zhu2025transformers,chen2025stronger}, which lacks the mathematical capacity to map extreme outliers into sparse, near-constant vectors. \cref{table:normalization} summarizes these results.
\begin{table}[H]
\vspace{-3pt}
  \caption{\textbf{Normalization configuration ablations.} Applying post-block normalization (Sandwich) or element-wise transformations (DynamicTanh) successfully suppresses massive activations. Notably, the model still maintains a significant sink ratio through alternative strategies, demonstrating that sinks can exist independently of massive activations. {\setlength{\fboxsep}{0pt}\colorbox{lime!10}{\strut Highlighted rows denote the baseline result.}}}
  \label{table:normalization}
  \begin{center}
    \begin{small}
      \begin{sc}
        \begin{tabular}{l@{\hspace{-2pt}}cccr}
          \toprule
            Setup & Perplexity & Sink Ratio & Spike \\
            \midrule
            \multicolumn{4}{c}{Normalization} \\
            \midrule
            Sandwich & 9.8 & 44.7\% & 520 \\
            Sandwich (QK)  & 10.0 & 42.0\% & 92 \\
            DynamicTanh & 10.0 & 61.0\% & 153 \\
            \rowcolor{lime!10} Pre-Norm & 10.1 & 46.0\% & 3818 \\
          \bottomrule
        \end{tabular}
      \end{sc}
    \end{small}
  \end{center}
  \vspace{-14pt}
\end{table}

The results demonstrate that normalization serves as a direct lever to decouple sinks from spikes. Sandwich normalization reduces spikes while preserving a sink ratio nearly identical to the baseline.
Because the extra $\text{RMSNorm}$ layer bounds the block output, it prevents the residual stream from accumulating the unbounded values necessary for massive outliers.
Replacing the block-level norm with QKNorm almost entirely eliminates spikes, confirming that these outliers are primarily generated to influence the query and key projections.

Conversely, as observed by \citet{owen2025refined}, element-wise transformations like DynamicTanh also prevent the emergence of massive activations entirely. This aligns with our hypothesis: because DynamicTanh is bounded and operates element-wise rather than via a vector-wide norm, it cannot facilitate the creation of sparse vectors from high-magnitude spikes. Interestingly, the DynamicTanh model yields the highest sink ratio while maintaining low spike magnitudes. Inspection of attention patterns reveals that the model still designates the first token as a stable reference point, achieving this through alternative strategies rather than magnitude-driven normalization. These results confirm that while massive spikes are an artifact of specific normalization configurations, they are not a prerequisite for attention sinks.

\subsection{Ablating Attention Sinks}

Building on our earlier analysis, we find that sink formation depends critically on whether sink and non-sink keys can occupy geometrically separable subspaces. We therefore begin by ablating per-head representational capacity, which determines whether the attention subspace has sufficient room to segregate sink keys from non-sink keys. We then conduct two further ablations motivated by prior work: one on gated attention~\cite{qiu2025gated}, which has been shown to reduce attention sinks and massive activations, and one on context length, motivated by~\cite{xiao2024duoattention}, which argues that attention sinks primarily bias short-range dependence. We ablate all three factors in turn below.

\subsubsection{Attention Head settings}
Our earlier findings identified the segregation of sink and non-sink keys as the primary driver of sink formation. Since this mechanism is inherently tied to per-head capacity, we systematically ablate total head count, head dimension, and head factorization to disentangle their individual contributions. Results are summarized in~\cref{table:attention_setup}.

The results confirm that \emph{head dimension} is the dominant architectural factor governing sink emergence. Increasing $\d{head}$ from $8$ to $128$ produces a monotonic rise in both the sink ratio and spike magnitude, supporting our geometric hypothesis: larger head dimensions expand the attention subspace sufficiently to cleanly separate sink keys from non-sink keys, enabling the generation of a large logit gap.

When total attention capacity ($\d{head} \times \N{head}$) is held fixed, concentrating it into \emph{fewer, larger} heads consistently strengthens sink behavior and improves perplexity, suggesting a potential link between sink ratio and model performance. Conversely, increasing the number of heads at fixed $\d{head}$ yields only marginal gains in the sink ratio, indicating that sink formation saturates once sufficient per-head capacity is available and that distributing capacity across more heads yields diminishing returns.
\begin{table}[H]
  \caption{\textbf{Attention head settings ablations.} Head dimension is the primary architectural driver of sink formation; larger dimensions provide the capacity for the attention subspace to isolate the sink keys. Concentrating capacity into fewer, larger heads intensifies sink behavior. {\setlength{\fboxsep}{0pt}\colorbox{lime!10}{\strut Highlighted rows denote the baseline configuration.}}}
  \label{table:attention_setup}
  \begin{center}
    \begin{small}
      \begin{sc}
        \begin{tabular}{lcccr}
          \toprule
            Setup & Perplexity & Sink Ratio & Spike \\
            \midrule
            \multicolumn{4}{c}{Number of Heads} \\
            \midrule
            $8$  & 10.4 & 37.1\% & 1253  \\
            $16$ & 10.3 & 41.7\% & 1936  \\
            \rowcolor{lime!10} $32$ & 10.1 & 46.0\% & 3818 \\
            \midrule
            \multicolumn{4}{c}{Head Dimension} \\
            \midrule
            $8$   & 11.3 & 4.1\%  & 291  \\
            $16$  & 10.8 & 9.8\%  & 315  \\
            $32$  & 10.5 & 27.9\% & 829  \\
            $64$  & 10.3 & 37.7\% & 2112 \\
            \rowcolor{lime!10} $128$ & 10.1 & 46.0\% & 3818 \\
            \midrule
            \multicolumn{4}{c}{Head Dim / Number of Heads} \\
            \midrule
            $8/512$   & 10.7 & 11.0\% & 1205 \\
            $16/256$  & 10.4 & 30.8\% & 1750 \\
            $32/128$  & 10.3 & 41.1\% & 1916 \\
            $64/64$   & 10.2 & 44.1\% & 2523 \\
            \rowcolor{lime!10} $128/32$ & 10.1 & 46.0\% & 3818 \\
            $256/16$  & 10.1 & 52.1\% & 3429 \\
          \bottomrule
        \end{tabular}
      \end{sc}
    \end{small}
  \end{center}
\end{table}

\subsubsection{Gated Attention}
\label{subsubsecton:gated_attention}

Following \cite{qiu2025gated}, we employ gated attention variants to test the hypothesis that dynamic multiplicative routing can destabilize or prevent attention sink formation. As shown in \cref{table:gated_attention}, gating conditioned on the current hidden representation drastically suppresses the sink ratio and effectively eliminates spikes, with minimal impact on perplexity. By contrast, unconditional gating or gating tied to static signals (such as position or token embedding) preserves strong sink behavior.
\begin{table}[h!]
  \caption{\textbf{Gated attention ablations.} Conditional gating---where the gate is a function of current representation---eliminates the need for attention sinks when applied per channel or per head. This suggests that sinks function as a ``learned gate'' to balance head contributions. Unconditional or static gates (positional or token-based) fail to suppress sinks, as they lack the dynamic, input-dependent routing necessary to substitute for sink behavior.}
  \label{table:gated_attention}
  \centering
  \begin{center}
    \begin{small}
      \begin{sc}
      \begin{tabular}{l@{\hspace{-8pt}}ccc}
          \toprule
            Setup & Perplexity & Sink Ratio & Spike \\
            \midrule
            \multicolumn{4}{c}{Conditional Gating} \\
            \midrule
            Channel & 10.0 & 4.5\% & 202 \\
            Head & 10.1 & 6.4\% & 186 \\
            Single & 10.2 & 31.2\% & 316 \\
            \midrule
            \multicolumn{4}{c}{Unconditional Gating} \\
            \midrule
            Channel & 10.1 & 42.2\% & 1922 \\
            Head & 10.1 & 41.3\% & 1884 \\
            Single & 10.2 & 44.3\% & 1797 \\
            \midrule
            \multicolumn{4}{c}{Conditional on Static Signal} \\
            \midrule
            Positional & 10.1 & 41.1\% & 1755 \\
            Token Embedding & 10.0 & 31.1\% & 1966 \\
          \bottomrule
        \end{tabular}
      \end{sc}
    \end{small}
  \end{center}
\end{table}
Among the conditional gating configurations, per-channel and per-head gates both eliminate attention sinks entirely. A single gate per token, however, yields elevated sink ratios and slightly higher perplexity---consistent with our earlier finding that sink formation is a head-level phenomenon. Under unconditional gating, the static gate fails to suppress either attention sinks or massive activations. Similarly, gating conditioned on positional or token embeddings does not eliminate sinks, as these signals are fixed and cannot adapt to the evolving context.

Taken together, these results suggest that attention sinks serve as a form of implicit input-conditioned gating: the effective routing behavior depends on the prompt history rather than being a fixed property of a particular head, position, or token. When the model has access to a dynamic, representation-conditioned gate, it can modulate attention routing on the fly, eliminating the structural need to maintain a spike token via large residual spikes.

\subsubsection{Training Context Length}

\citet{xiao2024duoattention} suggest that attention sinks facilitate short-range dependence in sink heads. Consistent with this, we observe that sink heads predominantly attend to nearby tokens of the query. We therefore vary the training context-length distribution to test whether sinks are an inductive bias of short-range training, controlling the distribution by adjusting the range of sequence positions over which the training loss is computed. Results are shown in \cref{table:context_length}.

\begin{table}[h!]
  \caption{\textbf{Context-length ablations.} Attention sinks are largely induced to facilitate short-context prediction. When the training distribution is restricted to long sequences, the sink ratio collapses, indicating that sinks are primarily utilized to support short-range dependence. {\setlength{\fboxsep}{0pt}\colorbox{lime!10}{\strut Highlighted rows denote the baseline configuration.}}}
  \label{table:context_length}
  \begin{center}
    \begin{small}
      \begin{sc}
        \begin{tabular}{l@{\hspace{-2pt}}cccr}
          \toprule
            Setup & Perplexity & Sink Ratio & Spike \\
            \midrule
            \multicolumn{4}{c}{Context Length (min/max)} \\
            \midrule
            1/256  & 12.4 & 42.1\% & 5411  \\
            1/1024 & 10.6 & 46.3\% & 4442  \\
            \rowcolor{lime!10} 1/4096 & 10.1 & 46.0\% & 3818 \\
            \midrule
            1024/4096 & 10.1 & 13.0\% & 38470 \\
            1024/5120 & 10.1 & 8.0\%  & 42365 \\
            2048/4096 & 10.6 & 1.2\%  & 7193  \\
            2048/6144 & 10.0 & 5.8\%  & 30634 \\
          \bottomrule
        \end{tabular}
      \end{sc}
    \end{small}
  \end{center}
\end{table}

When the training distribution includes short sequences, the sink ratio remains stable regardless of the maximum context length. Removing short contexts entirely---optimizing only over long-range positions---causes the sink ratio to collapse dramatically. This confirms that attention sinks are fundamentally a byproduct of short-context training: in mixed-length regimes, the first token provides a cheap, universally available global reference that to reduce the influence of far away tokens. Excluding short-context positions from the training loss therefore reveals that the majority of sinks are induced specifically to facilitate local prediction within a global attention mechanism, corroborating the role of sink heads identified by \citet{xiao2024duoattention}.

\subsection{Summary of Findings}

Our ablation study reveals three critical insights into the nature of massive activations and attention sinks:

\begin{enumerate}
    \item \textbf{Causal independence of spikes:} While spikes and sinks often co-occur, they are not inextricably linked. Normalization techniques like Sandwich Norm or QKNorm can eliminate massive activations without destroying the attention sink. This suggests that spikes are an artifact of the Pre-Norm architecture's tendency to accumulate unbounded values, which the model then exploits---but does not strictly require---to create logit contrast.
    \item \textbf{Sinks as a gating mechanism:} The disappearance of sinks in the presence of Conditional Gating suggests that attention sinks are a learned workaround. In the absence of an explicit gate to modulate information flow, the model repurposes the first token as a numerical ``dumping ground'' to effectively gate off unnecessary attention heads.
    \item \textbf{Context-length induction:} Sinks are fundamentally driven by the need to model short-range dependencies using a global attention mechanism. By dumping attention into the first token, the model can effectively ignore long-range context when it is not predictive, a behavior that becomes unnecessary when the model is trained exclusively on long-context sequences.
\end{enumerate}

\subsection{Discussion}

Our ablations paint a coherent picture of how massive activations and attention sinks arise, interact, and can be independently controlled. Across optimization hyperparameters, feed-forward designs, normalization configurations, and attention settings, the two phenomena respond differently to the same interventions — suggesting that their frequent co-occurrence in standard LLMs reflects incidental architectural interactions rather than a deep functional coupling.

\paragraph{Normalization as the bridge between spikes and sinks.} Normalization emerges as the central architectural link between the two phenomena. Standard pre-norm RMSNorm allows unbounded residual values to accumulate and maps spike tokens into sparse, near-constant vectors, which in turn provide a stable substrate for sink formation. Yet this link is incidental rather than necessary: sandwich normalization and DynamicTanh both suppress spikes while leaving a robust sink ratio intact, confirming that sinks can find alternative strategies when magnitude-driven normalization is unavailable. Mechanistically, massive activations interact with normalization to function as implicit parameters.

\paragraph{Attention sinks as a learned routing strategy.} Sink formation is independently driven by two factors: the dimensionality of the per-head attention subspace, which determines whether sink and non-sink keys can be geometrically separated, and the training context-length distribution, which establishes whether attention sinks are useful. Conditional gating experiments reinforce this view — sinks serve as an implicit, input-dependent routing mechanism that biases certain heads toward local, short-range dependencies, and one the model abandons as soon as an explicit dynamic gate is provided.

\paragraph{Independent suppression without performance cost.} Crucially, each phenomenon can be suppressed in isolation without measurable degradation in language modeling performance. This separation has practical implications: architectural choices that eliminate spikes for inference efficiency need not disrupt the short-range routing behavior that sinks provide, and vice versa. Their overlap in standard pretrained LLMs is best understood as a byproduct of the default normalization and training recipe, not a reflection of any underlying functional necessity.
\section{Related Work}
\label{section:related_work}

\noindent
\textbf{Attention sinks} have been observed across various Transformer models of all sizes~\citep{xiao2023efficient, yu2024unveiling}, including LLMs, vision-language models~\citep{zhang2026drives}, and multimodal models~\citep{cappellazzo2025mitigating}. Early evidence of this phenomenon traces back to the identification of outlier dimensions in BERT-scale models~\citep{kovaleva2021bert}, with their emergence linked to the frequency of tokens in training data~\citep{puccetti2022outlier}. They are not considered a downstream-task artifact~\citep{owen2025refined}: they emerge during pre-training and persist through instruction tuning~\citep{gu2024attention}. Proposed explanations span multiple hypotheses~\citep{liu2024intactkv, sandoval2025using, zhangattention, shang2025forgetting}, with many attributing attention sinks to softmax normalization in self-attention~\citep{gu2024attention, guo2024active, miller2023, velivckovic2024softmax, xiao2023efficient, lin2025look}.

Recent work has characterized the functional role of sinks, identifying ``dormant''~\citep{sandoval2025using} or ``garbage''~\citep{sok2026garbage} attention heads dominated by sink tokens that act as redundant dumping grounds. This structure has been exploited for efficient inference through ``streaming'' heads~\citep{xiao2024duoattention}, adaptive KV cache eviction that preserves first tokens~\citep{ge2023model}, layer-condensed cache strategies~\citep{wu2024layer}, and hybrid sparse attention patterns~\citep{fu2025h}. However, the disproportionate attention captured by boundary sinks also leads to low information retrieval for the middle part of long contexts~\citep{liu2023lost}. \medskip\\

\noindent
\textbf{Massive activations} were first identified in LLMs as extreme outlier features concentrated in specific channels~\citep{dettmers2022gpt3} and have been systematically characterized and shown to co-locate with attention sink tokens later~\citep{sun2024massive}. These outliers become increasingly pronounced as models scale~\citep{ahmadian2023intriguing}. They display behavior analogous to implicit bias terms: their magnitudes are stable across inputs, and they are tightly coupled with ``massive weights'' whose perturbation causes performance collapse~\citep{oh2024house}. These activations occupy fixed, largely input-agnostic dimensions and can be induced by highly aligned ``super'' weights~\citep{yu2024super}.

The presence of massive activations poses significant challenges for low-precision serving and training. Outlier channels severely degrade quantization performance~\citep{wei2022outlier, bondarenko2021understanding}, necessitating specialized techniques such as per-token scaling~\citep{yao2022zeroquant}, mixed-precision decomposition~\citep{dettmers2022gpt3, zhao2024atom, huang2024slim}, and outlier migration via shifting or Hadamard transformations~\citep{wei2023outlier, xi2023training, wang2025bitnet}. Furthermore, these activations can amplify numerical rounding errors during inference~\citep{budzinskiy2025numerical} and complicate ultra-low-precision FP4 training~\citep{abecassis2025pretraining}. \medskip\\

\noindent
\textbf{Mitigations and unifying theories} that jointly explain these phenomena have recently emerged. Training-time mitigations primarily modify the attention mechanism or normalization layers. Alternatives to softmax include \textit{sigmoid}~\citep{gu2024attention, ramapuram2024theory}, \textit{ReLU}~\citep{guo2024active}, \textit{softmax-off-by-one}~\citep{kaul2024attention, miller2023}, and Elastic-Softmax~\citep{fu2026attention}. Theoretical analysis of LayerNorm suggests it drives outlier emergence through re-centering and re-scaling~\citep{xu2019understanding}; consequently, normalization-free architectures~\citep{chen2025stronger} and ``Outlier Protected'' blocks~\citep{he2024understanding} have been proposed to eliminate the structural incentive for sinks.

Proactive strategies to suppress outliers include spectral sphere constraints~\citep{xie2026controlled}, smooth architectural modifications~\citep{liang2025tweo}, and explicit weight rescaling during pre-training~\citep{owen2025b}. Alternative routing mechanisms, such as hybrid attention-SSM architectures~\citep{dong2024hymba} or gated networks~\citep{yang2024gated}, provide explicit capacity control that reduces reliance on sinks. These findings complement theories like Mix-Compress-Refine~\citep{queipo2025attention}, which posits that massive activations drive the ``compression valley''~\citep{skean2025layer} and induce spectral dominance. Joint emergence has also been attributed to adaptive optimization dynamics, motivating orthogonalized variants such as OrthoAdam~\citep{kaul2024attention}.
\section{Conclusion}

This research clarifies the mechanistic relationship between massive activations and attention sinks in large language models. The findings demonstrate that these phenomena, while often co-occurring, are not inherently linked but are instead decoupled architectural artifacts of the pre-norm Transformer design. By identifying the specific roles of normalization and residual accumulation, the study shows that massive activations function as global implicit parameters, whereas attention sinks serve as local modulators for attention heads.

Furthermore, the evidence suggests that both phenomena can be independently mitigated through alternative architectural choices—such as modifying the normalization configuration—without sacrificing language-modeling performance. This separation of functional roles provides a clearer path for future model optimization, particularly for improving efficiency in quantization, pruning, and long-context inference. Ultimately, these insights move beyond descriptive observations to offer a structural understanding of how internal LLM representations are shaped by specific training and design decisions.


% \newpage
% \clearpage

\bibliography{main}
\bibliographystyle{icml2026}

%%%%%%%%%%%%%%%%%%%%%%%%%%%%%%%%%%%%%%%%%%%%%%%%%%%%%%%%%%%%%%%%%%%%%%%%%%%%%%%
%%%%%%%%%%%%%%%%%%%%%%%%%%%%%%%%%%%%%%%%%%%%%%%%%%%%%%%%%%%%%%%%%%%%%%%%%%%%%%%
% APPENDIX
%%%%%%%%%%%%%%%%%%%%%%%%%%%%%%%%%%%%%%%%%%%%%%%%%%%%%%%%%%%%%%%%%%%%%%%%%%%%%%%
%%%%%%%%%%%%%%%%%%%%%%%%%%%%%%%%%%%%%%%%%%%%%%%%%%%%%%%%%%%%%%%%%%%%%%%%%%%%%%%
\newpage
\appendix
\onecolumn
\crefalias{section}{appendix}

\section{Experimental Settings}
\label{appendix:experimental_settings}

ll models in \cref{section:ablations} are trained on the DCLM dataset~\citep{li2024datacomp} using a shared codebase and a common baseline recipe. 
Unless otherwise noted, we keep the training pipeline fixed and vary only the factors under study in each ablation.
The default recipe closely follows the Llama-style pretraining setup~\citep{touvron2023llama}.
With some additional details that we couldn't find from the original paper, we refer to the \textit{torchtitan} \cite{liang2025torchtitan} and \textit{Olmo} \cite{olmo2025olmo} codebase.
Tables~\ref{table:model_config} report the baseline architecture and optimization hyperparameters used across experiments.

\begin{table}[h]
\centering
\caption{\textbf{Baseline training configurations.}}
\label{table:model_config}
\small
\begin{tabular}{lc}
\toprule
\textbf{Hyperparameter} & \textbf{Value} \\
\midrule
Layers & 32 \\
Hidden size & 4096 \\
Attention heads & 32 \\
Head dimension & 128 \\
Intermediate size & 11008 \\
Vocabulary size & 32000 \\
\midrule
Optimizer & AdamW \\
$\beta_1, \beta_2$ & 0.9, 0.95 \\
Weight decay & 0.1 \\
Gradient clipping & 1.0 \\
Base learning rate & $\num{3.0e-4}$ \\
LR schedule & Cosine decay \\
Warmup steps & 2{,}000 \\
Training steps & 5{,}000 \\
Batch size & 2M tokens \\
Min LR ratio & 0.1 \\
\bottomrule
\end{tabular}
\end{table}

\noindent\textbf{Datasets.} During evaluation, we randomly sample text from the C4 corpus~\citep{raffel2020exploring}, with a total budget of up to $1024\times4096$ tokens. We tokenize the sampled text and partition it into fixed-length chunks of $\{64, 256, 1024, 2048, 4096\}$ tokens, selecting the chunk size to match the configured context window for each model.

\noindent\textbf{Sink Ratio.} For an input sequence of length $T$, let $A^{l,h}\in[0,1]^{T\times T}$ denote the (causal) self-attention matrix at layer $l$ and head $h$, where $A^{l,h}_{t,k}$ is the attention weight from query position $t$ to key position $k$. Following~\citet{gu2024attention}, we define the \emph{importance score} of position $k$ as the attention it receives on average:
\begin{align}
\alpha^{l,h}_k \;\coloneqq\; \frac{1}{T}\sum_{t=1}^{T} A^{l,h}_{t,k}.
\end{align}
A head is said to exhibit an attention sink (at threshold $\epsilon$) if there exists a position in the first half of the sequence that receives more than $\epsilon$ average attention, i.e.,
$\max\limits_{1\le k \le T/2}\alpha^{l,h}_k > \epsilon$. We then define the \emph{sink ratio} for this sequence as the fraction of heads (averaged over heads across all layers) that satisfy this criterion:
\begin{align}
s_{\epsilon}
\;=\;
\frac{1}{LH}\sum_{l=1}^{L}\sum_{h=1}^{H}
\mathds{1}\!\left(\max_{1\le k \le [T/2]}\alpha^{l,h}_k > \epsilon\right).
\end{align}
Finally, we report the model-level sink ratio by averaging $s_{\epsilon}$ over evaluation sequences. In our experiments, we use $\epsilon=0.3$ and $T=64$ consistently.

\section{Theorems and Derivations}
\label{appendix:mathematical_proof}

\begin{theorem}[\textbf{Attention output as a sum over heads}]
\label{theorem:attn_output}
Let $\O^{(h)}\coloneqq\A^{(h)} \V^{(h)} \in \R^{T \times \d{head}}$ 
denote the output of head $h$ with $\V^{(h)}\!\coloneqq\!\tilde{\H} \W{V}^{(h)}$ and $\tilde{\H}$ being the hidden representations of inputs right before attention block. Let $\W{O} \in \R^{(\N{head} \cdot \d{head}) \times \d{model}}$ be the output projection weight matrix and partition it by head as
    \begin{align}
        \W{O} = 
        \Concat\!\left({\W{O}^{(1)}}, \dots, {\W{O}^{(\N{head})}}\right)^\top
    \end{align}
with $\W{O}^{(h)} \in \R^{\d{head} \times \d{model}}$, $h\in\{1, \dots, \N{head}\}$.

Then the attention block output
    \begin{align}
        \Concat\!\left(\O^{(1)},\dots,\O^{(\N{head})}\right)\cdot\W{O}
    \end{align}
can be written as
    \begin{align}
        \sum_{h=1}^{\N{head}} \A^{(h)} \tilde{\H} \W{V}^{(h)} \W{O}^{(h)}.
    \end{align}
\end{theorem}

\begin{proof}
By definition of each head,
    \begin{align}
        \O^{(h)} = \A^{(h)}\V^{(h)} = \A^{(h)}\,\tilde{\H}\W{V}^{(h)}.
    \end{align}
Using the head-wise block partition of $\W{O}$, multiplying the concatenation
by $\W{O}$ decomposes additively:
    \begin{align}
        \Concat(\O^{(1)},\dots,\O^{(\N{head})})\cdot\W{O}
        = \sum_{h=1}^{\N{head}}\O^{(h)}\W{O}^{(h)}.  
    \end{align}
Substituting $\O^{(h)} = \A^{(h)}\tilde{\H}\W{V}^{(h)}$ into the above equation yields
    \begin{align}
        & \Concat\left(\O^{(1)},\dots,\O^{(\N{head})}\right)\W{O}=\sum_{h=1}^{\N{head}} \A^{(h)}\,\tilde{\H}\W{V}^{(h)}\,\W{O}^{(h)}.
    \end{align}
\end{proof}

\begin{theorem}[\textbf{Quadratic-form approximation of a SiLU feed-forward coordinate}]\label{theorem:silu_quadratic_form_approx}
Under the approximation of SiLU function (\cref{equation:SwiGLU_approximation}) for spike tokens with input representations $\tilde{\h}$, the output of the feed-forward block $\y$ has the quadratic form approximation for coordinate $k$ as follows
\begin{align}
\y_k = \tilde \h^\top \U_k \tilde \h,
\end{align}
where
\begin{align}
\U_k \coloneqq \sum_{i=1}^{d_{\mathrm{ff}}} \W{down}^{(k,i)}\; \W{gate}^{(i)} \W{up}^{(i)\top}.
\end{align}
\end{theorem}

\begin{proof}
Let $\z \in \mathbb{R}^{\d{hidden}}$ denote the intermediate hidden state
\begin{equation}
    \z \coloneqq \W{gate} \tilde{\h} \odot \W{up} \tilde{\h}.
\end{equation}
The $i$-th element of $\z$, i.e. $\z_i$ could be represented as
\begin{align}
    \z_i &= ({\W{gate}^{(i)\top}} \tilde{\h}) \cdot(\W{up}^{(i)\top} \tilde{\h})\nonumber \\
    &= \tilde{\h}^\top \W{gate}^{(i)} \W{up}^{(i)\top} \tilde{\h},
\end{align}
where $\W{gate}^{(i)}$ and $\W{up}^{(i)}$ are the $i$-th rows of the weight matrices. The term $\W{gate}^{(i)}{\W{up}^{(i)}}^\top$ is equal to the outer product of two vectors, forming a rank-1 matrix $\V_i$. We can now rewrite the hidden element as the quadratic form
\begin{equation}
    \z_i = \tilde{\h}^\top \V_i \tilde{\h}.
\end{equation}
Next, consider the $k$-th element of the output vector $\y$, denoted $\y_k$. By the linearity of the final projection ($\W{down}$):
\begin{align}
    \y_k &= \sum_{i} \W{down}^{(k,i)} \z_i \nonumber \\
    &= \sum_{i} \W{down}^{(k,i)} \cdot \left(\tilde{\h}^\top \V_i \tilde{\h}\right) \nonumber \\
    &= \tilde{\h}^\top\left( \sum_{i} \W{down}^{(k,i)} \V_i \right)\tilde{\h}.
\end{align}
Therefore, we conclude that the output element is in quadratic form
\begin{equation}
    \y_k = \tilde{\h}^\top \U_k \tilde{\h},
\end{equation}
where the matrix $\U_k$ is defined as the weighted sum of the rank-1 components
\begin{equation}
    \U_k =\sum_{i} \W{down}^{(k,i)} \V_i= \sum_{i} \W{down}^{(k,i)}\cdot\left(\W{gate}^{(i)} \W{up}^{(i)\top}\right).
\end{equation}
This derivation establishes the form given in~\cref{equation:quadratic_form} and concludes the proof.
\end{proof}

\begin{theorem}[\textbf{Coordinate bound under RMS normalization}]
\label{theorem:rms_norm_bound}
Let $\h\in\R^{\d{model}}$ be any non-zero vector and define its RMS-normalized version
\begin{align}
\tilde{\h} \;\coloneqq\; \RMSNorm(\h) \;=\; \sqrt{\d{model}}\,\frac{\h}{\|\h\|_2}.
\end{align}
Then every coordinate of $\tilde{\h}$ is bounded in magnitude by $\sqrt{\d{model}}$:
\begin{align}
\big|\tilde{\h}_i\big| \;\le\; \sqrt{\d{model}}, \quad \forall i\in\{1,\dots,\d{model}\}.
\end{align}
\end{theorem}

\begin{proof}
For any coordinate $i\in\{1,\dots,\d{model}\}$, by definition,
\begin{align}
\big|\tilde{\h}_i\big|
= \sqrt{d_{\mathrm{model}}}\,\frac{|\h_i|}{\|\h\|_2}.
\end{align}
Since 
\begin{align}
    \|\h\|_2^2=\sum_{j=1}^{\d{model}}\h_j^2 \ge \h_i^2,
\end{align}
we have $\|\h\|_2 \ge |\h_i|$.
Therefore,
\begin{align}
\frac{|\h_i|}{\|\h\|_2}\le 1
\quad\Longrightarrow\quad
\big|\tilde{\h}_i\big|
\le \sqrt{d_{\mathrm{model}}}.
\end{align}
\end{proof}

\section{Additional Empirical Results}
\label{appendix:additional_empirical_results}

\noindent\textbf{Open-Source Models.} While our primary analysis focused on Llama 2 7B, we validate the universality of our findings across the diverse set of open-source models detailed in~\cref{table:model_list}. These models span multiple families (Llama~2, Llama~3, Qwen2.5, Qwen3), depths (28 to 48 layers), and parameter counts (7B to 14B).
\begin{table}[H]
    \centering
    \vspace{-0.1in}
    \caption{\textbf{List of open source models evaluated in~\cref{appendix:additional_empirical_results}.} We observe consistent massive activations and attention sinks phenomena across the following diverse model families and sizes.}
    \label{table:model_list}
    \small
    \begin{tabular}{cccccc}
        \toprule
        Model Family & Model Name & Layers & Dimensions & Heads & Huggingface Model Id \\ \midrule
        \multirow{3}{*}{\makecell{\textsc{Llama}}} & Llama 2 7B  & 32 & 4096 & 32 & \texttt{meta-llama/Llama-2-7b-hf}  \\
        & Llama 2 13B & 40 & 5120 & 40 & \texttt{meta-llama/Llama-2-13b-hf} \\ 
        & Llama 3 8B & 32 & 4096 & 32 & \texttt{meta-llama/Meta-Llama-3-8B} \\\midrule
        \multirow{4}{*}{\makecell{\textsc{Qwen}}} & Qwen2.5 7B  & 28 & 3584 & 28 & \texttt{Qwen/Qwen2.5-7B}  \\ 
        & Qwen2.5 14B & 48 & 5120 & 40 & \texttt{Qwen/Qwen2.5-14B} \\
        & Qwen3 8B & 36 & 4096 & 32 & \texttt{Qwen/Qwen3-8B} \\
        & Qwen3 14B & 40 & 5120 & 40 & \texttt{Qwen/Qwen3-14B} \\
        \bottomrule
    \end{tabular}
\vspace{-0.2in}
\end{table}

\paragraph{Universality of Step-Up/Step-Down Dynamics.}
\cref{figure/appendix:magnitudes} visualizes the top-3 coordinate magnitudes through the residual stream for all 12 evaluated models. We observe two consistent behaviors across all architectures:
\begin{enumerate}
    \item \textbf{Massive Activations:} Every model exhibits activation spikes orders of magnitude larger than the baseline variance.
    \item \textbf{Feed Forward Block Driven Origin:} Comparing the ``post-residuals'' (top panels) and ``block outputs'' (bottom panels) plots in~\cref{figure/appendix:magnitudes} confirms that these spikes are not merely accumulated residual error. Instead, they originate abruptly at specific ``step-up'' blocks and are neutralized by subsequent ``step-down'' blocks, matching the mechanism described in \cref{section:from_spikes_to_sinks}.
\end{enumerate}
\vspace{-5pt}
\begin{figure}[H]
  \begin{center}
    \begin{subfigure}[h]{.49\textwidth}
    \centerline{\includegraphics[height=0.48\linewidth,trim={0 2.0cm 0 0},clip]{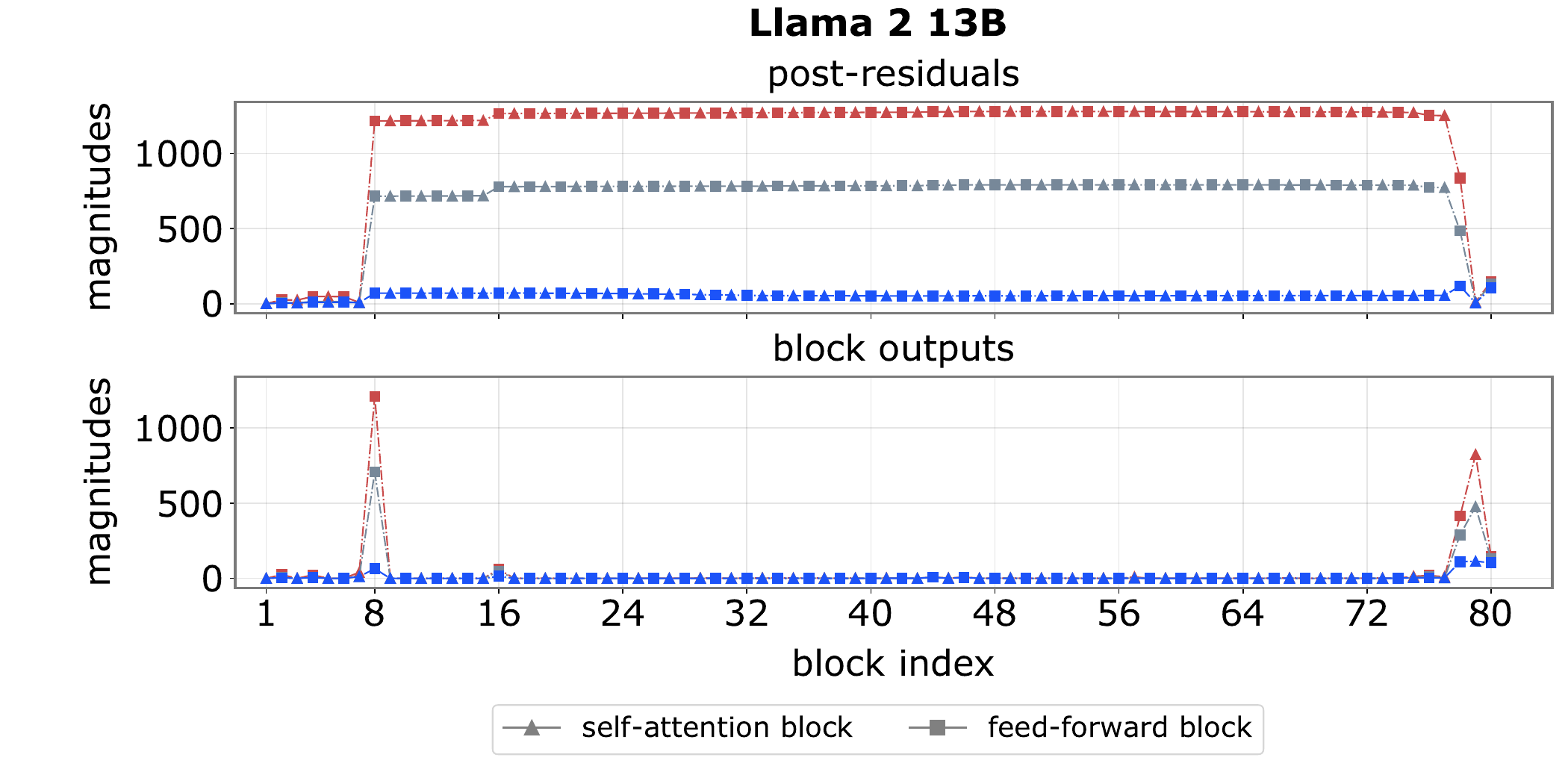}}
    \end{subfigure}
    \hfill
    \begin{subfigure}[h]{.49\textwidth}
    \centerline{\includegraphics[height=0.48\linewidth,trim={3.2cm 2.0cm 0 0},clip]{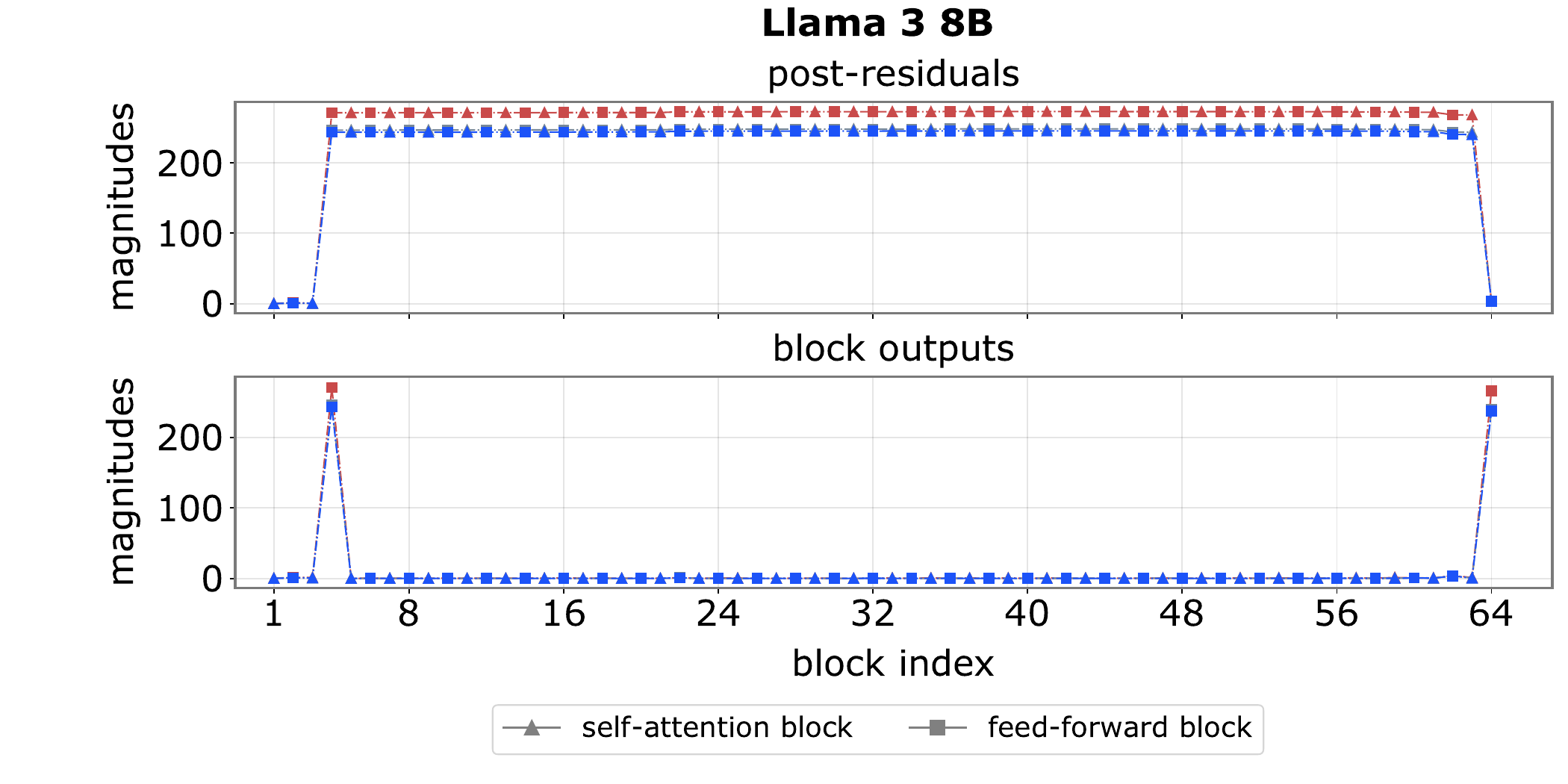}}
    \end{subfigure}
    
    \begin{subfigure}[h]{.49\textwidth}
    \centerline{\includegraphics[height=0.48\linewidth,trim={0 2.0cm 0 0},clip]{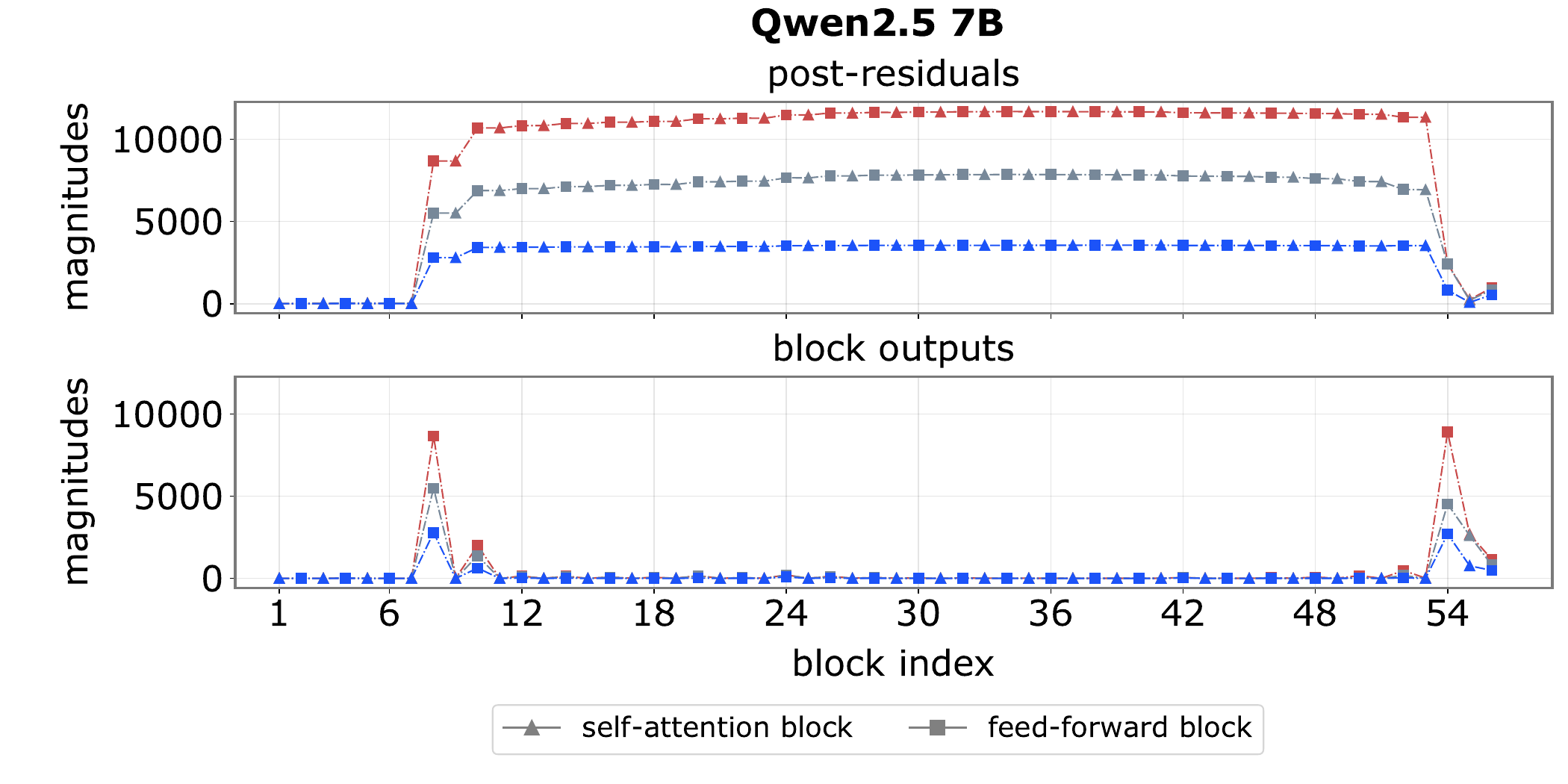}}
    \end{subfigure}
    \hfill
    \begin{subfigure}[h]{.49\textwidth}
    \centerline{\includegraphics[height=0.48\linewidth,trim={2.2cm 2.0cm 0 0},clip]{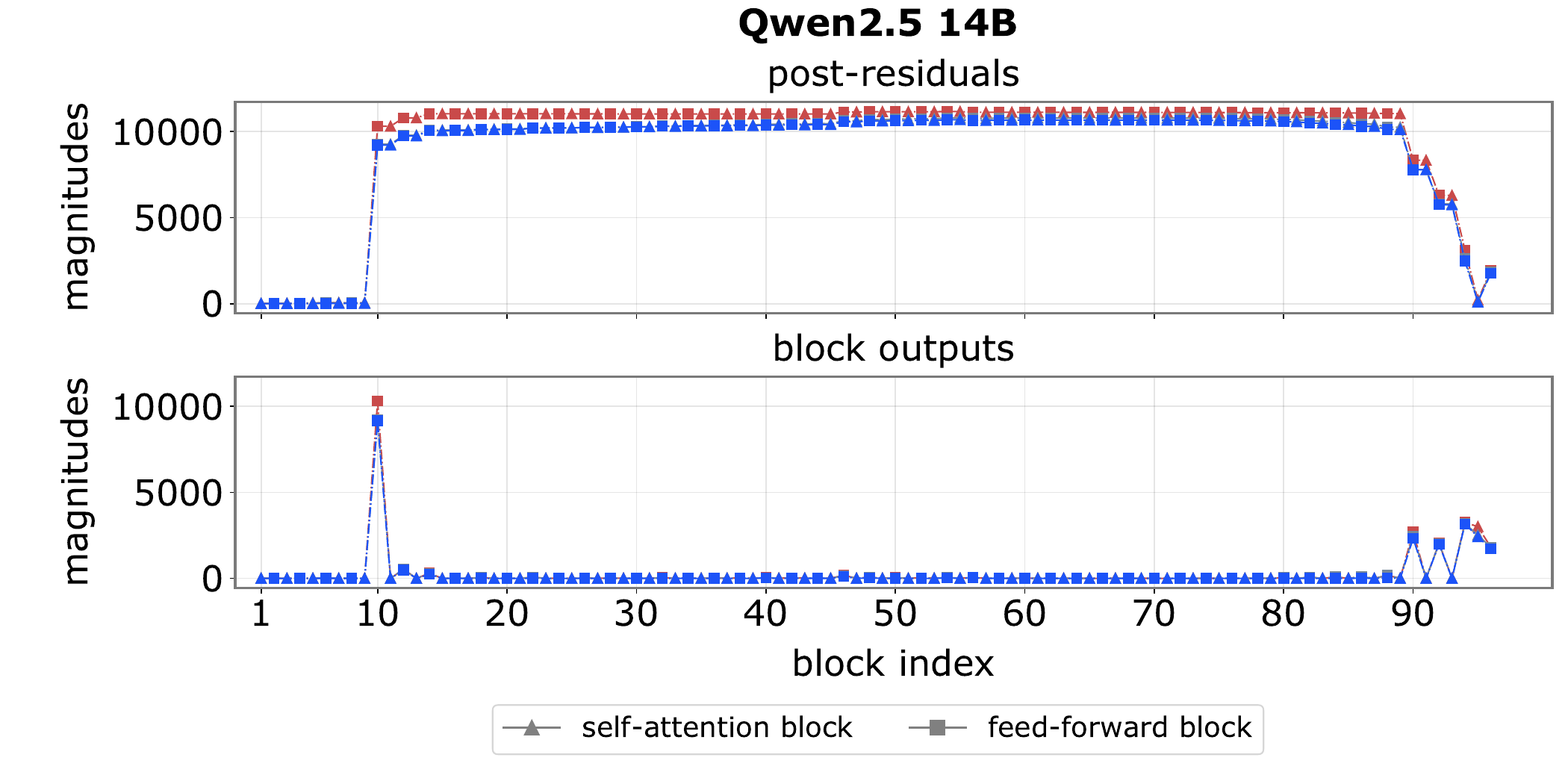}}
    \end{subfigure}
    
    \begin{subfigure}[h]{.49\textwidth}
    \centerline{\includegraphics[height=0.55\linewidth]{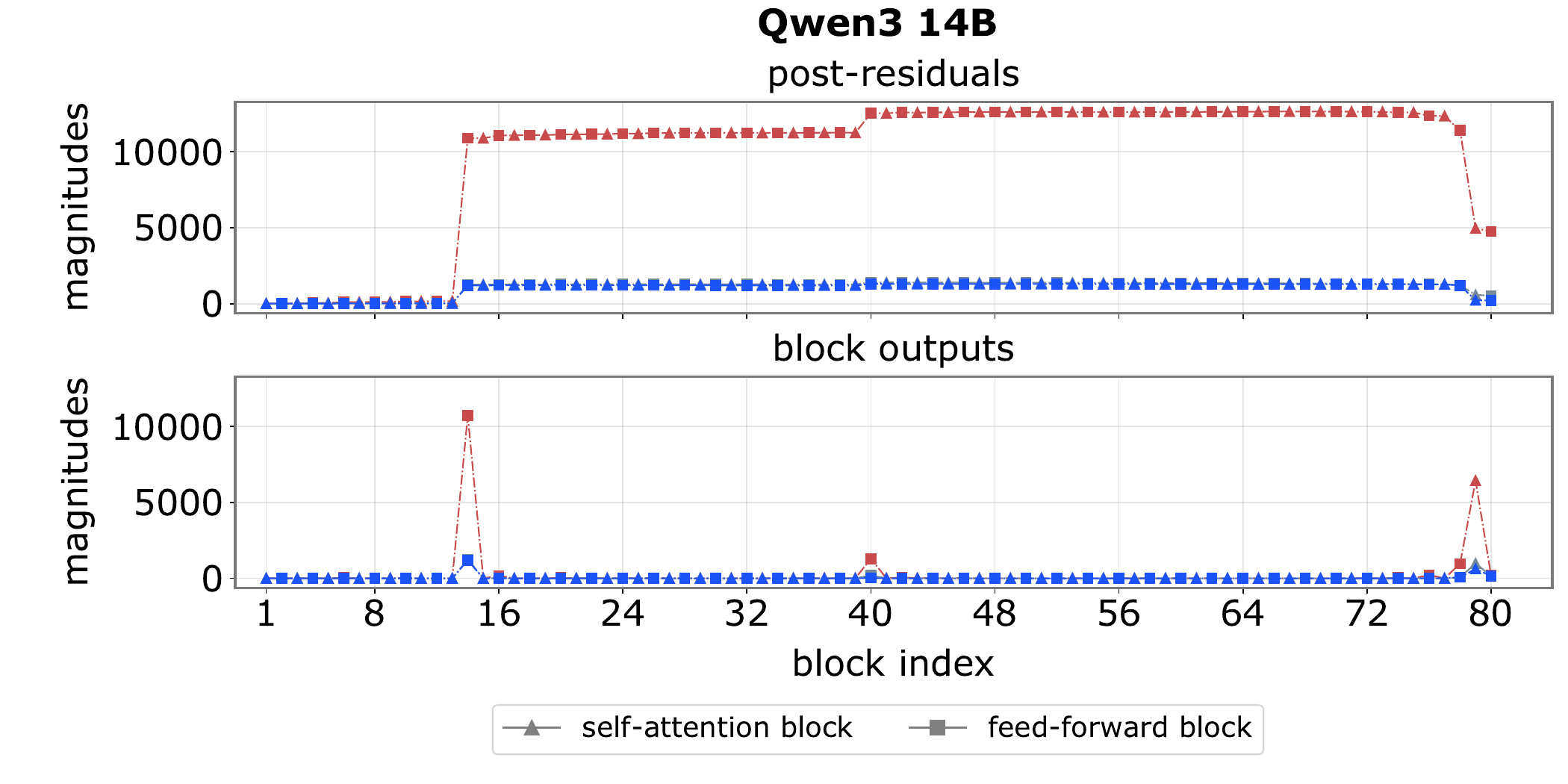}}
    \end{subfigure}
    \hfill\null
    \vspace{-5pt}
    \caption{\textbf{Top-3 coordinate magnitudes after (top panels) and before (bottom panels) the residual branch for 12 open source models.} All models exhibit the characteristic step-up and step-down behavior driven by one or few early blocks and then the late blocks.}
    \label{figure/appendix:magnitudes}
  \end{center}
\vspace{-18pt}
\end{figure}

\paragraph{Universality of Frobenius norm outliers.}
To confirm that these activations arise from the directional quadratic amplification mechanism, we analyze the Frobenius norms of the quadratic form matrices $\U_k$. As shown in~\cref{figure/appendix:frobenius_norms}, channels exhibiting massive activations correspond to $\U_k$ matrices with exceptionally large Frobenius norms.

For example, in Llama 3 8B, distinct spikes are visible at channels 788, 1384, and 4062, which align with the massive activation spikes observed in~\cref{figure/appendix:magnitudes}. This confirms that the alignment of attention sinks with high-gain quadratic directions in the MLP is a structural invariant across Llama model families.

\begin{figure}[H]
  \begin{center}
    \begin{subfigure}[b]{.49\textwidth}
    \centerline{\includegraphics[height=0.6\linewidth]{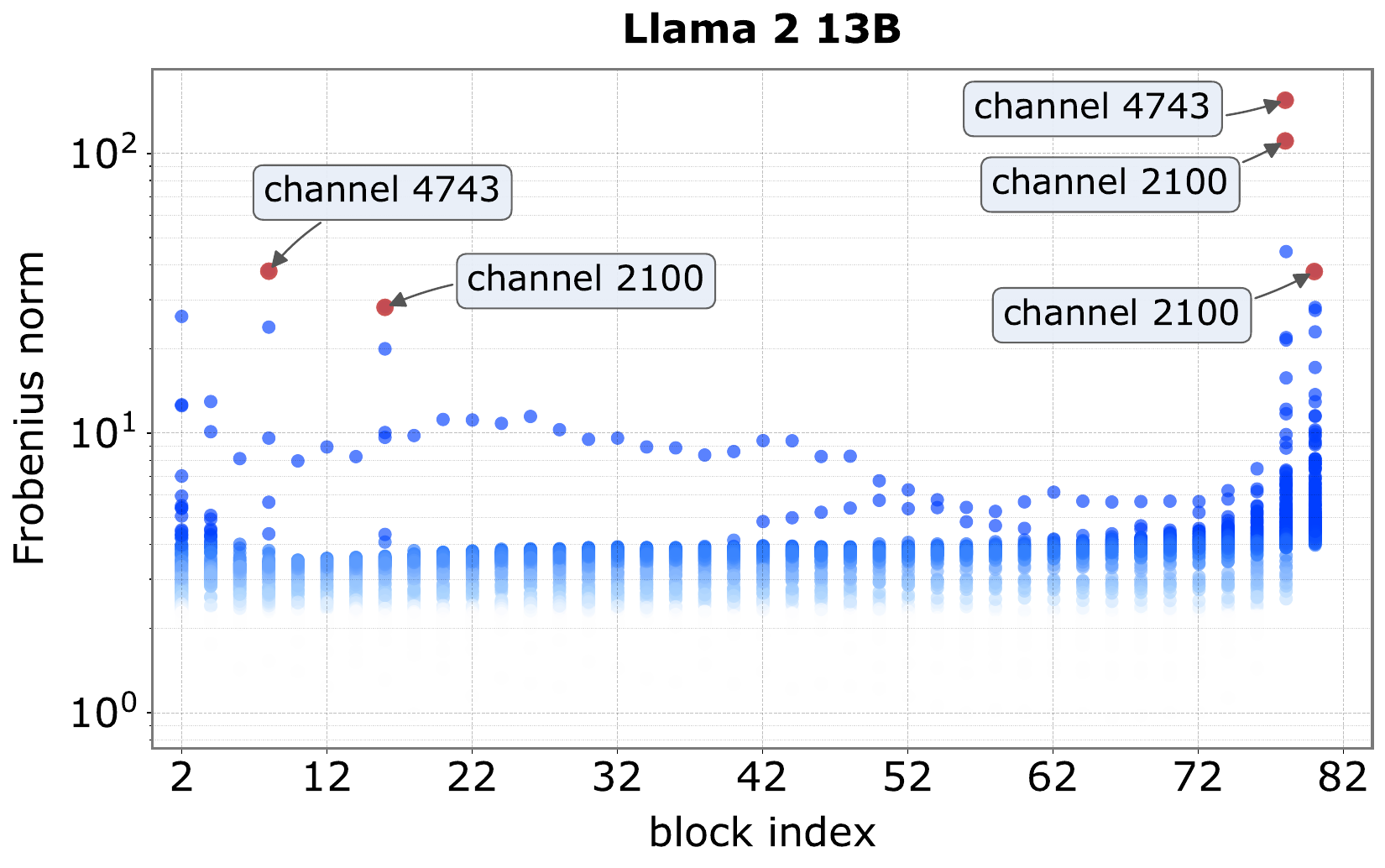}}
    \end{subfigure}
    \hfill
    \begin{subfigure}[b]{.49\textwidth}
    \centerline{\includegraphics[height=0.6\linewidth,trim={1.2cm 0 0 0},clip]{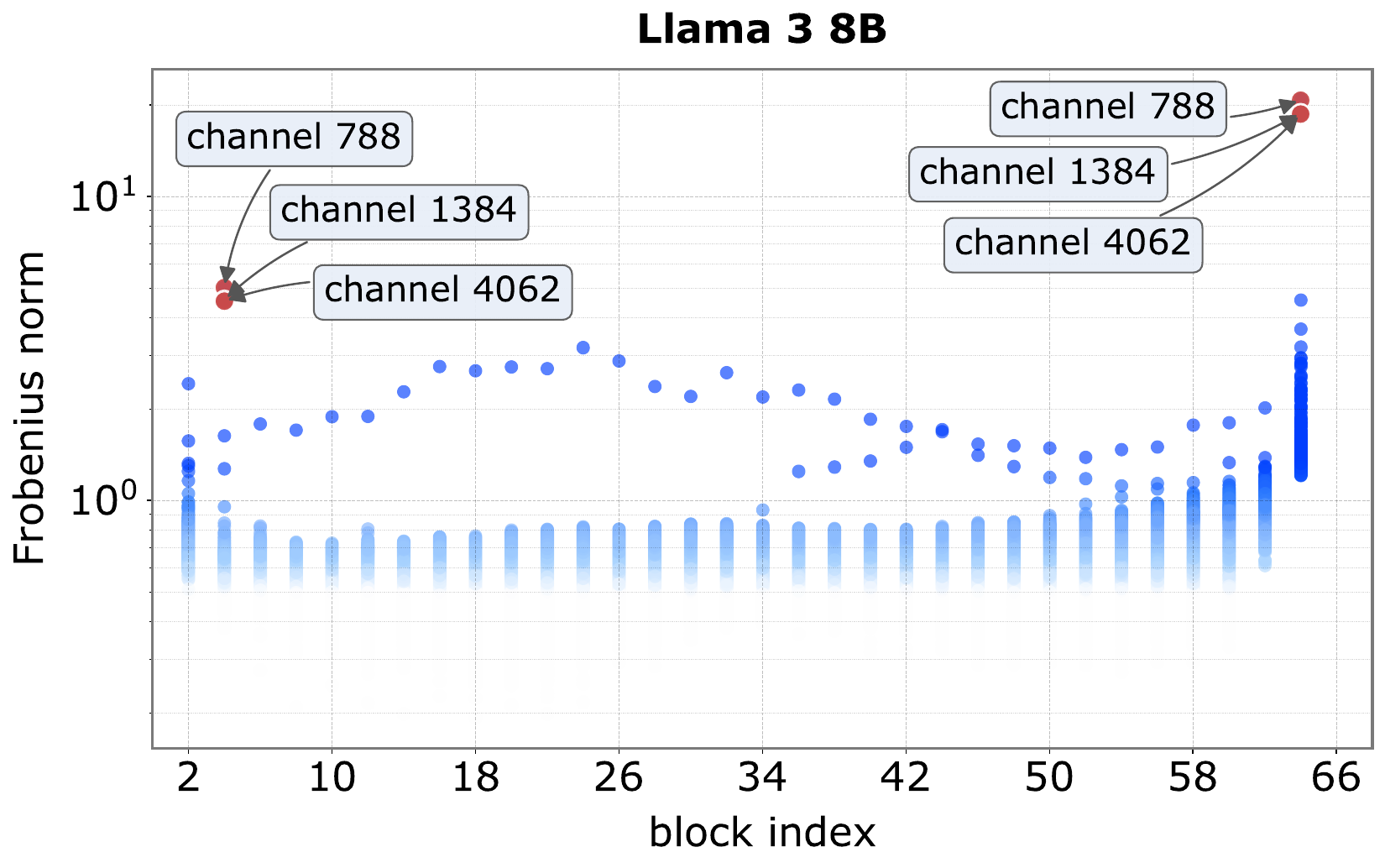}}
    \end{subfigure}
    \caption{\textbf{Frobenius norms $\|\U_k\|_F$ for the quadratic forms of Llama models.} Spike channels align with $\U_k$ matrices that exhibit substantially larger norms than typical channels. These high-norm coordinates appear exclusively in step-up and step-down blocks.}
    \label{figure/appendix:frobenius_norms}
  \end{center}
\vspace{-18pt}
\end{figure}

%%%%%%%%%%%%%%%%%%%%%%%%%%%%%%%%%%%%%%%%%%%%%%%%%%%%%%%%%%%%%%%%%%%%%%%%%%%%%%%
%%%%%%%%%%%%%%%%%%%%%%%%%%%%%%%%%%%%%%%%%%%%%%%%%%%%%%%%%%%%%%%%%%%%%%%%%%%%%%%

\end{document}